\documentclass{article}

\usepackage{mathrsfs}
\usepackage{amssymb}
\usepackage{amsmath,amsthm}
\usepackage{enumerate}   
\usepackage{microtype}
\usepackage{graphicx}
\usepackage{subfigure}
\usepackage{todonotes}
\usepackage{booktabs} 
\usepackage{hyperref}
\usepackage{cleveref}
\usepackage{mathtools}
\usepackage{bbm}
\usepackage{csquotes}
\usepackage[skip=0pt]{caption}

\let\P\relax\newcommand{\P}{\mathbbm{P}}

\newcommand*{\N}{\mathbb{N}}
\DeclareMathOperator*{\E}{\mathbbm{E}}
\DeclareMathOperator*{\V}{\mathbbm{V}}
\DeclareMathOperator{\R}{\mathbb{R}}

\theoremstyle{plain}
\newtheorem{theorem}{Theorem}[section]
\newtheorem{lemma}[theorem]{Lemma}

\newtheorem{proposition}[theorem]{Proposition}
\newtheorem{assumption}{Assumption}
\theoremstyle{definition}
\newtheorem{definition}[theorem]{Definition}

\theoremstyle{remark}
\newtheorem{remark}[theorem]{Remark}

\usepackage[accepted]{icml2022}
\icmltitlerunning{Robust SDE-Based Variational Formulations for Solving Linear PDEs via Deep Learning}

\begin{document}

\twocolumn[
\icmltitle{Robust SDE-Based Variational Formulations\\for Solving Linear PDEs via Deep Learning}
\icmlsetsymbol{equal}{*}

\begin{icmlauthorlist}
\icmlauthor{Lorenz Richter}{equal,dida,zib,ber}
\icmlauthor{Julius Berner}{equal,vie}
\end{icmlauthorlist}

\icmlaffiliation{vie}{Faculty of Mathematics, University of Vienna, Austria}
\icmlaffiliation{ber}{Institute of Mathematics, Freie Universit\"at Berlin, Germany}
\icmlaffiliation{zib}{Zuse Institute Berlin, Germany}
\icmlaffiliation{dida}{dida Datenschmiede GmbH, Germany}
\icmlcorrespondingauthor{Julius Berner}{\href{mailto:julius.berner@univie.ac.at}{julius.berner@univie.ac.at}}
\icmlcorrespondingauthor{Lorenz Richter}{\href{mailto:lorenz.richter@fu-berlin.de}{lorenz.richter@fu-berlin.de}}

\icmlkeywords{ICML, Kolmogorov PDE, Robustness, SDE, Deep Learning, Machine Learning, Numerical Analysis}

\vskip 0.3in
]

\printAffiliationsAndNotice{\icmlEqualContribution (the author order was determined by \texttt{numpy.random.rand(1)})}

\begin{abstract}
The combination of Monte Carlo methods and deep learning has recently led to efficient algorithms for solving partial differential equations (PDEs) in high dimensions. Related learning problems are often stated as variational formulations based on associated stochastic differential equations (SDEs), which allow the minimization of corresponding losses using gradient-based optimization methods. In respective numerical implementations it is therefore crucial to rely on adequate gradient estimators that exhibit low variance in order to reach convergence accurately and swiftly. In this article, we rigorously investigate corresponding numerical aspects that appear in the context of linear Kolmogorov PDEs. In particular, we systematically compare existing deep learning approaches and provide theoretical explanations for their performances. Subsequently, we suggest novel methods that can be shown to be more robust both theoretically and numerically, leading to substantial performance improvements.
\end{abstract}

\section{Introduction}
\label{sec: introduction}

In this paper we suggest novel methods for solving high-dimensional linear Kolmogorov PDEs. In particular, those methods allow for more efficient approximations by reducing the variance of loss estimators, as illustrated by numerical experiments like the one displayed in \Cref{fig: loss std first page}. 
\begin{figure}
  \centering
  \includegraphics[width=1.0\linewidth]{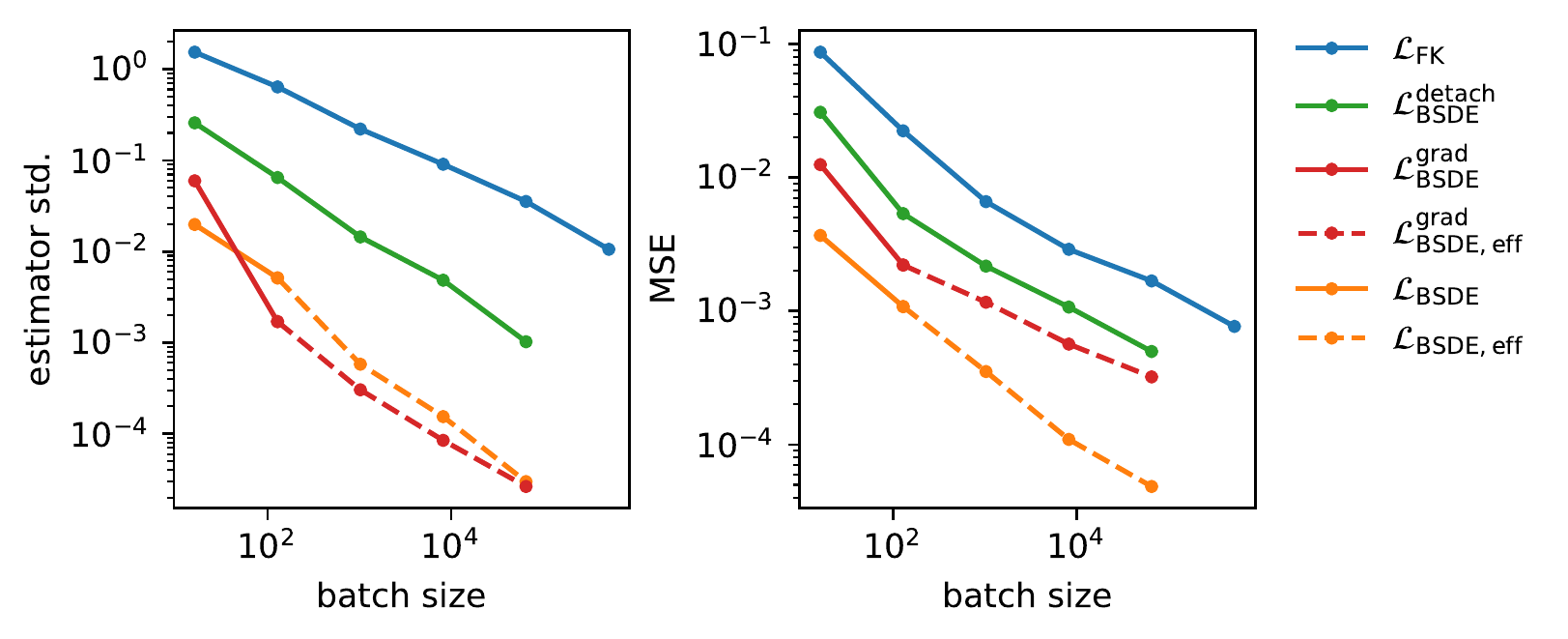}
  \vspace{-0.7cm}
  \caption{Standard deviation and performance (in terms of the mean squared error (MSE) on evaluation data) of different losses as a function of the batch size after training a neural network for $30k$ steps, where due to memory constraints different losses allow for different maximal batch size values.}
  \label{fig: loss std first page}
\end{figure}

High-dimensional PDEs are ubiquitous in applications appearing in economics, science, and engineering, however, their numerical treatment poses formidable challenges since traditional grid-based methods suffer from the curse of dimensionality. Recently, the combination of Monte Carlo methods and deep learning has led to feasible algorithms \cite{weinan2017deep,sirignano2018dgm,han2018solving, raissi2019physics,pfau2020ab,berner2021modern}, which, at least in some settings, are able to beat this curse \cite{jentzen2018proof, reisinger2019rectified, berner2020analysis}. The general idea is to rely on stochastic representations of the PDE which may be interpreted as variational formulations that can subsequently be minimized by iterative gradient decent methods \cite{nusken2021interpolating}. In respective algorithms it is therefore crucial to rely on adequate gradient estimators that exhibit low variance and eventually lead to robust numerical optimization routines, where robustness refers to the ability of coping with stochasticity in the optimization process.

In this article we focus on linear Kolmogorov PDEs for which variational formulations may be based on the celebrated Feynman-Kac formula \cite{beck2021solving}. This method is handy from an implementational point of view, however, we can demonstrate potential non-robustness issues that might deteriorate the numerical performance. To overcome these issues, we suggest alternative approaches resting on backward stochastic differential equations which lead to variance-reduced estimators and robust algorithms. Our contributions can be summarized as follows:
\begin{itemize}
    \item We provide a systematic numerical comparison of loss functions for solving linear Kolmogorov PDEs via deep learning.
    \item We can show both theoretically and numerically why certain methods are non-robust and might therefore incur numerical instabilities.
    \item We provide new loss functions and optimization methods which are provably robust and ultimately yield significant performance improvements in relevant high-dimensional numerical experiments.
\end{itemize}

\subsection{Related Work}
Solving linear Kolmogorov PDEs via the Feynman-Kac formula by means of deep learning has been suggested in \citet{beck2021solving} and further analyzed in \citet{berner2020numerically}. For an analysis of certain kinds of elliptic linear PDEs we refer to \citet{grohs2020deep}. Variational formulations of nonlinear PDEs based on backward stochastic differential equations (BSDEs) have been pioneered in \citet{weinan2017deep} and extensions that aim for approximations on entire domains have, for instance, been suggested in \citet{nusken2021interpolating}. We refer to \citet{Richter2021} for a comprehensive introduction and further perspectives, and to \citet{beck2019machine} for the treatement of fully nonlinear equations. An alternative method for the approximation of high-dimensional PDEs is termed \textit{physics-informed neural networks} (PINNs) \cite{raissi2019physics, sirignano2018dgm}, which relies on iterative residual minimizations. Note that in contrast to most SDE-based methods it relies on the explicit computation of Hessians and involves additional loss terms for the boundary values, which need to be tuned carefully.

In terms of variance reduction of Monte Carlo estimators we refer to \citet{vidales2018unbiased} for a control variate attempt via deep learning. Robustness and variance analyses of certain loss functions related to PDE problems have been conducted in \citet{nusken2021solving}. The related observation that including a certain It\^{o} integral term into the objective function can lead to variance-reduced estimators has been made in \citet{zhou2021actor}. Both works, however, consider different settings and focus only on Hamilton--Jacobi--Bellman (HJB) equations. Furthermore, they do not provide a detailed analysis of the interplay between robustness properties, performance, and computational feasibility, which is crucial for practical applications.

\subsection{Notation}
We assume that our random variables are defined on a suitable underlying filtered probability space $(\Omega,\mathcal{F},(\mathcal{F}_t)_t, \P)$ satisfying the usual conditions~\cite{klenke2013probability}. We denote the expectation of a random variable $A\colon\Omega \to \R$ by $\E[A]\coloneqq \int A \, \mathrm{d}\P$. Given another random variable $B\colon\Omega\to \R^k$ with $k\in\N$, we also consider the conditional expectation $
\E[A|B]$ of $A$ given $B$ and the conditional expectation $b \mapsto \E[A|B=b]$ of $A$ given $B=b$, see, e.g.,~\citet{ash2000probability} for the definitions. We further define the variance of $A$ by $\V[A]\coloneqq \E\left[(A-\E[A])^2\right]$ and the conditional variance of $A$ given $B$ by $\V[A|B]\coloneqq \E\left[(A-\E[A|B])^2| B\right]$. Throughout the article, $W$ is a standard $d$-dimensional Brownian motion which is adapted to the underlying filtration and we consider stochastic integrals w.r.t.\@ $W$ as, e.g., defined in~\citet{gall2016}. Further remarks on the notation can be found in \Cref{app: notation}.

\section{SDE-Based Variational Formulations of Linear PDEs}
\label{sec: main section}

In this paper we consider linear \emph{Kolmogorov partial differential equations}\footnote{PDEs of this type are often also referred to as \emph{Kolmogorov (backward) equations} and their adjoints are known as \emph{Fokker-Planck equations} or \emph{Kolmogorov forward equations.}} of the type
\begin{subequations}
\label{eq:PDE}
\begin{align}
    (\partial_t + L)V(x, t) &= 0, & (x, t) \in {\R}^d \times [0, T), \\
    V(x, T) &= g(x), & x \in {\R}^d,
\end{align}
\end{subequations}
where $g\in C(\R^d, \R)$ is a given terminal condition and
\begin{equation}
\label{eq:diff_operator}
        L \coloneqq \frac{1}{2} \sum_{i,j=1}^d (\sigma \sigma^\top)_{ij}(x,t) \partial_{x_i} \partial_{x_j} + \sum_{i=1}^d b_i(x,t) \partial_{x_i}
\end{equation}
is a differential operator based on given coefficient functions $\sigma \in C(\R^d \times [0, T], \R^{d\times d})$ and $b \in C(\R^d \times [0, T], \R^d)$. For simplicity, let us assume the existence of a unique strong solution $V \in C^{2,1}(\R^d\times [0, T], \R)$ to the PDE in~\eqref{eq:PDE} and refer to \Cref{as: conditions on b sigma g} in \Cref{app: proofs} for further technical details. At the same time, let us note that most of our results also hold for the more general concept of viscosity solutions that are continuous but not necessarily differentiable and require weaker conditions on the functions $b, \sigma$, and $g$, see~\citet{hairer2015loss}.
Furthermore, we can also consider more general linear PDEs as well as elliptic and parabolic problems on bounded domains, as elaborated on in \Cref{app: feynman-kac}. 

This covers a broad spectrum of PDEs, for which accurate and reliable solvers are of great importance to practitioners. For instance, such PDEs frequently appear in physics for the modelling of heat flow and diffusion processes \cite{widder1976heat, pascucci2005kolmogorov}. Moreover, the PDE in~\eqref{eq:PDE} includes the Black-Scholes equation and extensions thereof, used for pricing financial derivative instruments~\citep{black1973pricing, pironneau09, EKSTROM2010498}. Finally, we want to emphasize appearances of such PDEs in the field of machine learning, e.g., in the context of reinforcement learning~\cite{theodorou2010generalized, kiumarsi2017optimal} and diffusion-based generative modeling~\cite{huang2021variational}.

In order to design machine learning algorithms aiming to approximate solutions to \eqref{eq:PDE}, we consider loss functions
\begin{equation}
    \mathcal{L}: \mathcal{U} \to \R_{\ge 0},
\end{equation}
which shall be minimal if and only if $u \in \mathcal{U}$ fulfills $\eqref{eq:PDE}$, thereby offering a variational formulation that allows for iterative minimization strategies. Here, $\mathcal{U}$ is an appropriate function class holding sufficient approximation capacity and containing, for instance, a class of certain neural networks. Moreover, we implicitly assume that $\mathcal{U}$ is chosen such that $V\in \mathcal{U}$ and, for our theoretical results, we require the functions in $\mathcal{U}$ to be sufficiently smooth.

A first obvious choice for the loss could be the squared difference between the approximating function $u$ and the solution $V$, i.e.
\begin{equation} \label{eq:eval_loss}
    \mathcal{L}_\mathrm{Eval}(u) \coloneqq \E\left[\left(V(\xi,\tau) - u(\xi,\tau) \right)^2 \right],
\end{equation}
where $\xi$ and $\tau$ are suitable\footnote{In practice, one often chooses random variables that are uniformly distributed on a given hypercube. For our theoretical results, we assume that $\xi$ and $\tau$ are $\mathcal{F}_0$-measurable and exhibit sufficiently many bounded moments. This ensures that $\xi$ and $\tau$ are sufficiently concentrated and independent of the Brownian motion $W$.} random variables with values in $\R^d$ and $[0,T]$, respectively. At first glance, this loss seems intractable since the solution $V$ is just the quantity we are after and is therefore not known. However, let us recall the \emph{Feynman-Kac formula}, which connects the deterministic function $V$ to a stochastic process by stating that almost surely it holds that
\begin{equation}
\label{eq:fk}
    V(\xi,\tau) = \E\left[ g(X_T) | (\xi, \tau)\right],
\end{equation}
where $X$ is the solution to the SDE
\begin{equation}
\label{eq:SDE}
    \mathrm dX_s = b(X_s, s) \,\mathrm d s + \sigma(X_s, s) \,\mathrm d W_s, \quad X_\tau=\xi,
\end{equation}
see Appendix~\ref{app: feynman-kac} for details.
This allows us to rewrite~\eqref{eq:eval_loss} as
\begin{subequations}
\label{eq:eval_loss FK}
\begin{align}
    \mathcal{L}_\mathrm{Eval}(u) &= \E\left[\big(\E\left[ g(X_T) | (\xi, \tau)\right] - u(\xi,\tau) \big)^2 \right] \\
    \label{eq:mse_mc}
    &= \E\left[\E\left[ \Delta_u | (\xi, \tau)\right]^2 \right] \\
    \label{eq:eval_loss FK control variate}
    &= \E\left[\E\left[ \Delta_u -S_u | (\xi, \tau)\right]^2 \right],
\end{align}
\end{subequations}
which now relies on the (random) quantities\footnote{For the sake of readability, we omit several implicit dependencies. Specifically, the stochastic process $X$ and the quantities $\Delta_u$ and $S_u$ depend on the random variable $(\xi,\tau)$, defining the initial time and value of the SDE, i.e., $X_\tau = \xi$. 
Also note that the stochastic integral $S_u$ and the SDE in~\eqref{eq:SDE} are driven by the same Brownian motion $W$.
}
\begin{subequations}
\begin{align}
\label{eq: definition Delta_u}
    \Delta_u &\coloneqq g(X_T) - u(\xi, \tau), \\
    \label{eq: definition S_u}
    S_u &\coloneqq \int_\tau^T \big(\sigma^\top \nabla_x u\big) (X_s, s) \cdot \mathrm dW_s.
\end{align}
\end{subequations}
The last step \eqref{eq:eval_loss FK control variate} follows from the fact that, under mild regularity assumptions, the stochastic integral $S_u$ is a martingale which has zero expectation given $(\xi,\tau)$. Adding quantities with a known expectation is a common trick for variance reduction of corresponding estimators known as \textit{control variates}, see, e.g., Section 4.4.2 in \citet{robert2004monte}. The benefit of particularly
choosing $S_u$ becomes clear with the following Lemma, which states that for $u = V$ the corresponding control variate actually yields a zero-variance estimator.

\begin{lemma}[Optimal control variate]
\label{lem: zero variance}
Let $V$ be a solution to the PDE in \eqref{eq:PDE}. For $\Delta_V$ and $S_V$ as defined in \eqref{eq: definition Delta_u} and \eqref{eq: definition S_u} it almost surely holds that 
\begin{equation}
\label{eq:ito}
    \Delta_V = S_V,
\end{equation}
which implies that
$\V\left[\Delta_V - S_V | (\xi, \tau) \right] = 0$. 
\end{lemma}
\begin{proof}
The proof is an application of It\^{o}'s lemma. Together with the observation that $X$ is an It\^{o} process given by the SDE in~\eqref{eq:SDE}, it implies that almost surely it holds that
\begin{equation*}
  V(X_T,T) - V(\xi,\tau) = \int_{\tau}^T(\partial_t + L)V(X_s,s)\,\mathrm ds + S_V,
\end{equation*}
assuming that $V \in C^{2, 1}(\R^d \times [0, T], \R)$, see, e.g., Theorem 3.3.6 in \citet{karatzas1998brownian}. Using the fact that $V$ solves the PDE in~\eqref{eq:PDE}, this implies the statement.
\end{proof}

While, in principle, the new representation in \eqref{eq:eval_loss FK} makes the loss $\mathcal{L}_\mathrm{Eval}$ computationally tractable, an immediate disadvantage are the two expectations, which, for a numerical implementation, might be costly as they would both need to rely on Monte Carlo approximations.

It turns out, however, that the conditional expectations in~\eqref{eq:mse_mc} and~\eqref{eq:eval_loss FK control variate} are in fact not needed and we can further define the two losses 
\begin{equation}
\label{eq: FK loss}
    \mathcal{L}_\mathrm{FK}(u) \coloneqq \E\left[\Delta_u^2 \right] 
\end{equation}
and
\begin{equation}
\label{eq: BSDE loss}
    \mathcal{L}_\mathrm{BSDE}(u) \coloneqq \E\left[\left(\Delta_u -  S_u \right)^2 \right].
\end{equation}
Note that now the expectations in \eqref{eq: FK loss} and \eqref{eq: BSDE loss} are to be understood with respect to the random initial time $\tau$, random initial value $\xi$, and the randomness of the Brownian motion $W$, which defines the evolution of the stochastic process $X$ on the (random) interval $[\tau, T]$.

In the remainder of this article we will investigate the above losses with respect to certain robustness properties that will turn out to imply different statistical properties and consequently lead to different optimization performances of corresponding algorithms. Let us start with the following proposition which relates the three losses to one another and therefore guarantees that $\mathcal{L}_\mathrm{FK}$ and $\mathcal{L}_\mathrm{BSDE}$ indeed constitute meaningful objectives for approximating the solution $V$ to the PDE in~\eqref{eq:PDE}.
\begin{proposition}[Minima of $\mathcal{L}_\mathrm{FK}$ and $\mathcal{L}_\mathrm{BSDE}$]
\label{prop:opt}
    For every measurable $u\colon\R^d\times [0,T]\to\R$ it holds that
    \begin{subequations}
    \begin{align}
        \mathcal{L}_\mathrm{Eval}(u) &=  \mathcal{L}_\mathrm{FK}(u) - \V\left[S_V \right] \\
        &= \mathcal{L}_\mathrm{BSDE}(u) - \V\left[S_{V - u} \right].
    \end{align}
    \end{subequations}
    This implies that the solution $V$ to the PDE in~\eqref{eq:PDE} is the unique\footnote{Up to null sets w.r.t.\@ to the image measure of $(\xi,\tau)$.} minimizer of $\mathcal{L}_\mathrm{FK}$ and $\mathcal{L}_\mathrm{BSDE}$. The minima satisfy
    \begin{enumerate}[(i)]
        \item $\min_{u} \mathcal{L}_\mathrm{FK}(u) = \mathcal{L}_\mathrm{FK}(V) = \E[\Delta_V^2] = \V[S_V]$,
        \item $\min_{u} \mathcal{L}_\mathrm{BSDE}(u) = \mathcal{L}_\mathrm{BSDE}(V) =0$.
    \end{enumerate}
\end{proposition}
\begin{proof}
See \Cref{app: proofs}.
\end{proof}

\begin{remark}[Origin of the losses]
As already hinted at, the names of the losses originate from the Feynman-Kac (FK) formula and backward stochastic differential equations (BSDE), respectively. The former, as stated in \eqref{eq:fk} and detailed in \Cref{app: feynman-kac}, implies that the solution of the PDE in \eqref{eq:PDE} can be expressed as a conditional expectation, i.e.,
\begin{subequations}
\label{eq: FK contitional expectation}
\begin{align}
    V(x, t) &= \E[g(X_T) | X_t = x] \\
    &=\E[g(X_T) | (\xi, \tau)=(x,t)].
\end{align}
\end{subequations}
This can then be combined with the fact that the conditional expectation is the minimizer of a corresponding $L^2$-optimization problem and one readily recovers $\mathcal{L}_\mathrm{FK}$ as specified in \eqref{eq: FK loss}, see~\citet{beck2021solving}.
The latter loss, $\mathcal{L}_\mathrm{BSDE}$, originates from the BSDE
\begin{equation}
    \mathrm d Y_s = Z_s \cdot \mathrm dW_s, \quad Y_T = g(X_T),
\end{equation}
which can essentially be derived from It\^{o}'s lemma and which involves the backward processes $Y_s = V(X_s, s)$ and $Z_s = (\sigma^\top\nabla_x V)(X_s, s)$, see \Cref{app: BSDEs} for further details. 
The name of the loss $\mathcal{L}_{\mathrm{Eval}}$ refers to the fact that it is mostly used in order to evaluate the solution candidate $u$, using either a known solution $V$ or the reformulation in~\eqref{eq:eval_loss FK}. This is also how we evaluate the mean squared error (MSE) for our methods, see \Cref{app: computational details}.
\end{remark}

\subsection{Estimator Versions of Losses and Robustness Issues}
The expectations in the losses \eqref{eq:eval_loss}, \eqref{eq:eval_loss FK}, \eqref{eq: FK loss}, or \eqref{eq: BSDE loss} can usually not be computed analytically and we need to resort to Monte Carlo approximations $\mathcal{L}^{(K)}$ based on a given number $K \in \N$ of independent samples
\begin{equation*}
(\xi^{(k)},\tau^{(k)}, W^{(k)})\sim (\xi,\tau,W), \quad k=1,\dots, K.
\end{equation*}
In practice, one further has to employ a time discretization of the SDE, yielding the estimators
\begin{subequations}
\label{eq: loss estimators}
\begin{align}
\label{eq: FK loss estimator}
    \widehat{\mathcal{L}}^{(K)}_{\mathrm{FK}}(u) &= \frac{1}{K}\sum_{k = 1}^K \left(\widehat{\Delta}_u^{(k)}\right)^2, \\
   \label{eq: BSDE loss estimator}
    \widehat{\mathcal{L}}^{(K)}_{\mathrm{BSDE}}(u) &= \frac{1}{K}\sum_{k = 1}^K \left(\widehat{\Delta}_u^{(k)} - \widehat{S}_u^{(k)}\right)^2,
\end{align}
\end{subequations}
where $\widehat{\Delta}^{(k)}_u$ and $\widehat{S}_u^{(k)}$ are discretized versions of \eqref{eq: definition Delta_u} and \eqref{eq: definition S_u}, evaluated at the $k$-th sample.
More precisely, we define
\begin{subequations}
\label{eq: estimator discr}
\begin{align}
    \widehat{\Delta}^{(k)}_u &\coloneqq g(\widehat{X}^{(k)}_{J^{(k)}+1})-u(\xi^{(k)},\tau^{(k)}), \\
    \label{eq: bsde estimator discr}
    \widehat{S}_u^{(k)} &\coloneqq \sum_{j=1}^{J^{(k)}} s_j^{(k)},
\end{align}
\end{subequations}
where
\begin{equation}
    \label{eq: stoch integral discr incr}
   s_j^{(k)} \coloneqq (\sigma^\top \nabla_x u)(\widehat{X}_{j}^{(k)},t_j^{(k)})\cdot \big(W^{(k)}_{t_{j+1}} - W^{(k)}_{t_j}\big).
\end{equation}
Here, $(\widehat{X}^{(k)}_j)_{j=1}^{J^{(k)}+1}$ denotes a discretization of the solution to the SDE in~\eqref{eq:SDE} (driven by the Brownian motion $W^{(k)}$) on a time grid $ \tau^{(k)} = t^{(k)}_1 < \dots < t^{(k)}_{J^{(k)}+1} = T$ with initial condition $\widehat{X}^{(k)}_1=\xi^{(k)}$, see \Cref{app: computational details} for details.

In this section we will investigate different statistical properties of the Monte Carlo estimators. It will turn out that even though the loss $\mathcal{L}_\mathrm{FK}$ is prominent for solving linear PDEs in practice (see, e.g., \citet{beck2021solving, berner2020numerically}), it might have some disadvantages from a numerical point of view. In fact,~\Cref{lem: zero variance} and \Cref{prop:opt} already show that, while $\mathcal{L}_\mathrm{BSDE}$ can be minimized to zero expectation, $\mathcal{L}_\mathrm{FK}$ always exhibits the additional term
\begin{equation}
    \E[\Delta_V^2]=\V[\Delta_V] = \V[S_V]  =\E[S_V^2],
\end{equation}
which is a potential source of additional noise. By the It\^{o} isometry, this can also be written as 
\begin{equation}
    \E\left[S_V^2\right]=\E\left[\left(\int_\tau^T \|\sigma^\top \nabla_x V(X_s, s) \|^2\mathrm ds \right)    \right].
\end{equation}
The following proposition shows what this implies for the estimator versions of the losses at the optimum.

\begin{proposition}[Variance of losses]
\label{prop: variance of losses}
For the estimator versions of the losses defined in \eqref{eq: FK loss} and \eqref{eq: BSDE loss} it holds that
\begin{equation*}
    \E\left[ {\mathcal{L}}^{(K)}_{\mathrm{FK}}(V) \right] = \V\left[S_V\right],\,\,\,\,\,\,\, \V\left[ {\mathcal{L}}^{(K)}_{\mathrm{FK}}(V) \right] = \frac{1}{K} \V\left[S_V^2\right]
\end{equation*}
and 
\begin{equation*}
     \E\left[ {\mathcal{L}}^{(K)}_{\mathrm{BSDE}}(V) \right] = 0, \qquad\quad \V\left[ {\mathcal{L}}^{(K)}_{\mathrm{BSDE}}(V) \right] = 0.
\end{equation*}
\end{proposition}
 \begin{proof}
 See \Cref{app: proofs}.
 \end{proof}

\Cref{prop: variance of losses}  shows that neither the expectation nor the variance of $\mathcal{L}^{(K)}_{\mathrm{FK}}$ can be zero at the solution $u = V$ unless $S_V$ or $S_V^2$ are deterministic, which is usually not the case in practice as demonstrated by the numerical examples in~\Cref{sec: examples} and \Cref{fig: loss std evolution}. While the non-zero expectation of the estimator is typically not a problem for the optimization routine, the non-vanishing variance might cause troubles when converging to the global minimum. 

\begin{figure}
  \centering
  \includegraphics[width=1.0\linewidth]{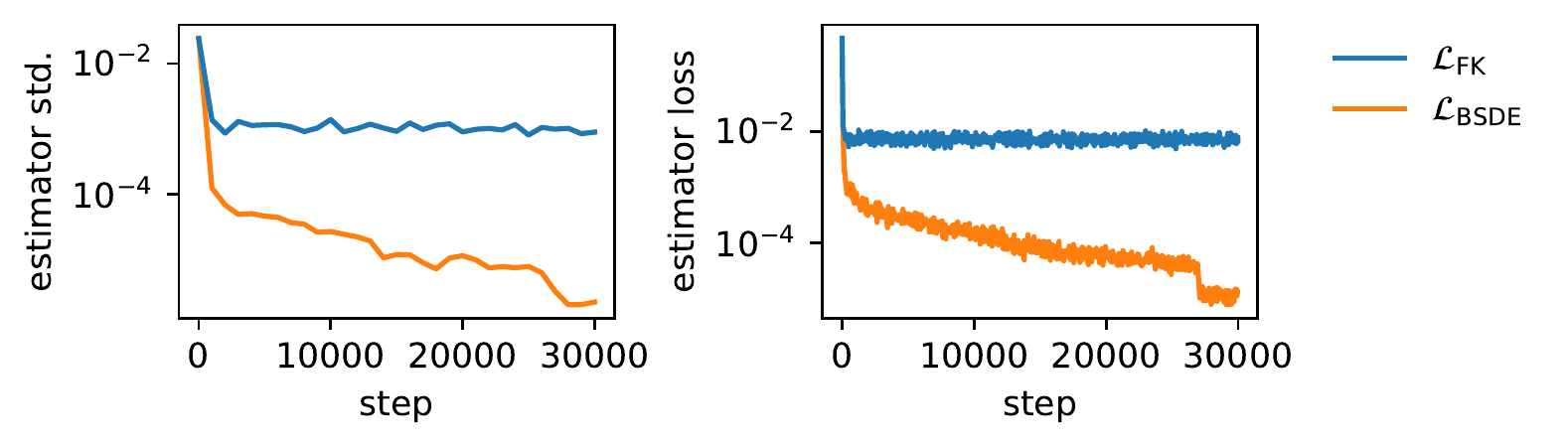}
  \vspace{-0.7cm}
  \caption{Estimator losses and their standard deviations as a function of the gradient steps when solving the HJB equation from \Cref{sec: hjb} with batch size $K = 128$. One observes that the statement of \Cref{prop: variance of losses} already holds in an approximative sense before converging to the solution $V$.}
  \label{fig: loss std evolution}
\end{figure}

Since the considered losses are usually minimized via variants of gradient decent, our further analysis shall therefore focus on the variance of corresponding gradient estimators. For this let us first define an appropriate notion of derivative.
\begin{definition}[Gateaux derivative]
We say that $\mathcal{L}\colon \mathcal{U} \to \R_{\ge 0}$ is \emph{Gateaux differentiable} at $u \in \mathcal{U}$ if for all $\phi\in \mathcal{U}$ the mapping
\begin{equation}
\varepsilon \mapsto  \mathcal{L}(u+ \varepsilon \phi)
\end{equation}
is differentiable at $\varepsilon=0$. 
The Gateaux derivative of $\mathcal{L}$ in direction $\phi$ is then defined as 
\begin{equation}
    \frac{\delta}{\delta u}\mathcal{L}(u; \phi) \coloneqq \frac{\mathrm d}{\mathrm d \varepsilon}\Big|_{\varepsilon=0} \mathcal{L}(u+ \varepsilon \phi).
\end{equation}
\end{definition}

Inspired by an analysis in \citet{nusken2021solving} let us first investigate the variance of derivatives at the solution $u = V$, where intuitively we might want to favor gradient estimators whose variances vanish at the global minimizer of the loss. The following proposition shows that this property is indeed only the case for $\mathcal{L}^{(K)}_\mathrm{BSDE}$.
\begin{proposition}[Gradient variances]
\label{prop:robust}
For the estimator versions of the losses defined in \eqref{eq: FK loss} and \eqref{eq: BSDE loss} and for all $\phi \in \mathcal{U}$ it holds that
\begin{enumerate}[(i)]
    \item $\V\left[\frac{\delta}{\delta u}{\Big|}_{u= V}\mathcal{L}^{(K)}_\mathrm{FK}(u; \phi)\right] = \frac{4}{K} \V\left[S_V \phi(\xi,\tau) \right]$,
    \item $\V\left[\frac{\delta}{\delta u}{\Big|}_{u= V}\mathcal{L}^{(K)}_\mathrm{BSDE}(u; \phi)\right] = 0$. 
\end{enumerate}
\end{proposition}
\begin{proof}
See \Cref{app: proofs}.
\end{proof}

Note that even though the gradient estimator $\mathcal{L}^{(K)}_\mathrm{FK}$ does not have vanishing variance at the solution, a large batch size $K$ can counteract this effect. We will study corresponding dependencies between variances and batch sizes in our numerical experiments in \Cref{sec: examples}.

While \Cref{prop:robust} only makes a statement on the gradient variances at the solution $u=V$, the following proposition shall indicate that, at least close to the solution, one might still expect a small variance of the gradients of $\mathcal{L}^{(K)}_\mathrm{BSDE}$, see also Figure~\ref{fig: loss grad std evolution}.

\begin{figure}
  \centering
  \includegraphics[width=1.0\linewidth]{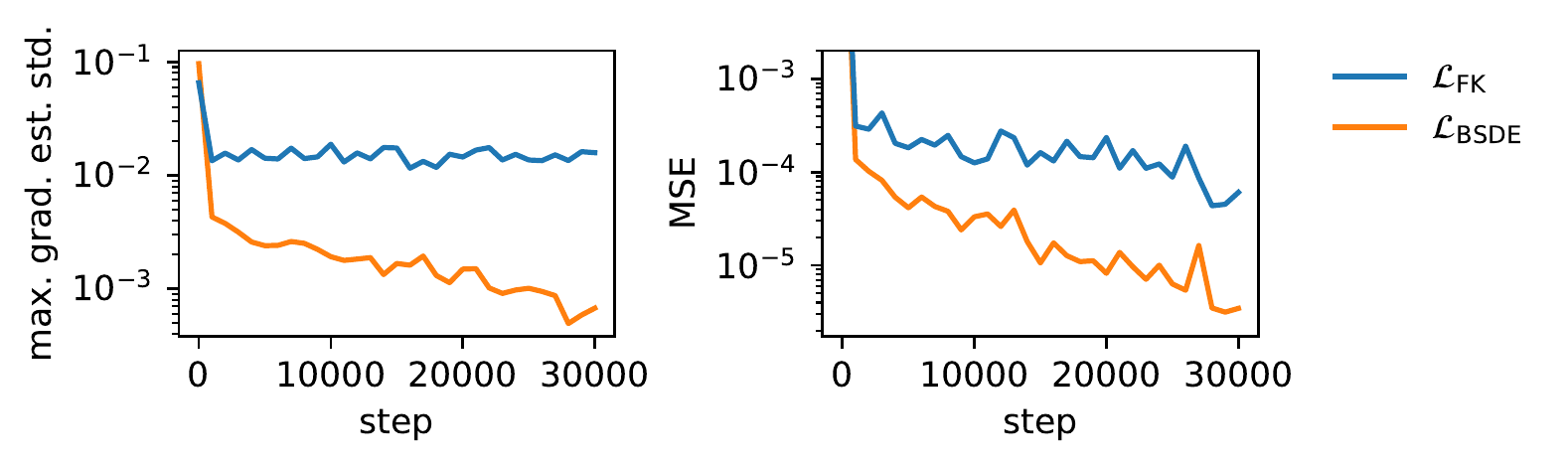}
  \vspace{-0.7cm}
  \caption{Gradient estimator standard deviations (maximum over directions $\phi = \partial_{ \theta_i}\Phi_{\theta}$, where $\Phi_{\theta}$ is a neural network with parameters $\theta$, see \Cref{app: computational details}) and the MSE (on evaluation data) as a function of the gradient steps when solving the HJB equation from \Cref{sec: hjb} with batch size $K=128$. In line with \Cref{prop:robust_around} the standard deviation of the gradient of $\widehat{\mathcal{L}}^{(K)}_\mathrm{BSDE}$ decreases when the solution is sufficiently well approximated.
  }
  \label{fig: loss grad std evolution}
\end{figure}

\begin{proposition}[Stability of $\mathcal{L}_\mathrm{BSDE}$ close to the solution]
\label{prop:robust_around}
Assume it holds almost surely that
\begin{equation}
\label{eq:robus_ass_1}
    | u(\xi, \tau) - V(\xi, \tau) | \le \varepsilon \quad \text{and} \quad  | \phi(\xi, \tau) | \le \kappa
\end{equation}
and that there exists $\gamma\in \R_{>0}$ such that for every $(x,t) \in \R^d\times [0,T]$ it holds that
\begin{equation}
\label{eq:robus_ass_2}
    \| \nabla_x u(x, t) - \nabla_x V(x, t) \| \le \varepsilon (1+\|x\|^\gamma) 
\end{equation}
and that
\begin{equation}
    \| \nabla_x \phi(x, t) \| \le \kappa (1+\|x\|^\gamma).
\end{equation}
Then the variance can be bounded by
\begin{equation}
    \V\left[\frac{\delta}{\delta u}\mathcal{L}^{(K)}_\mathrm{BSDE}(u; \phi)\right] \le \frac{C\varepsilon^2 \kappa^2}{K},
\end{equation}
where $C \in\R_{>0}$ is a constant that depends only on $\gamma$, $\xi$, $\tau$, and the given PDE in~\eqref{eq:PDE}.
\end{proposition}
 \begin{proof}
 See \Cref{app: proofs}.
 \end{proof}

\begin{remark}
\label{rem: approx grad}
We remark that the simultaneous approximation of a function and its gradient by a neural network, as required in~\eqref{eq:robus_ass_1} and~\eqref{eq:robus_ass_2}, has, for instance, been considered in~\citet{berner2019towards}.
Using \Cref{prop:opt} we can see that it is actually sufficient to approximate the gradient $\nabla_x V$ (which only implies the approximation of $V$ up to a constant) in order to reach zero variance. We also note that, for $u$ far away from the solution $V$, the term $S_u$ could potentially increase the estimator variance, see, e.g., \Cref{fig: bs detach} in the appendix.
\end{remark}

\subsection{Algorithmic Details and Neural Network Approximations}

So far our analysis has been agnostic of the choice of the approximating function space $\mathcal{U}$. In practice, we usually rely on functions represented by a neural network $\Phi_{\theta} \colon \R^{d+1}\to\R$ with parameters $\theta \in\R^p$. The training of such a network is then conducted by minimizing a suitable estimated and time-discretized loss $\widehat{\mathcal{L}}^{(K)}$ using some variant of gradient descent. In its simplest form, this corresponds to the update rule
\begin{equation}
\label{eq:sgd}
    \theta^{(m+1)} = \theta^{(m)} - \lambda \, \nabla_\theta \widehat{\mathcal{L}}^{(K)} (\Phi_{\theta^{(m)}}), 
\end{equation}
where $\lambda\in\R_{>0}$ is a learning rate and $\theta^{(0)}\in\R^p$ is a random initialization. Let us refer to \Cref{alg: algorithm} in \Cref{app: computational details} for further details.

Relating to the Gateaux derivatives from the previous section, let us note that we are particularly interested in the directions\footnote{We assume that the functions $\Phi_\theta$, $\theta \in \R^p$, as well as all partial derivatives $\partial_{ \theta_i}\Phi_\theta$ lie in $\mathcal{U}$.} $\phi=\partial_{ \theta_i}\Phi_\theta$ for $i \in \{1, \dots, p\}$. This choice is motivated by~\eqref{eq:sgd} and the chain rule of the Gateaux derivative, which, under suitable assumptions, states that
\begin{equation}
    \partial_{\theta_i} \widehat{\mathcal{L}}^{(K)}(\Phi_\theta) = \frac{\delta}{\delta u}{\Big|}_{u= \Phi_{\theta}}\widehat{\mathcal{L}}^{(K)}\left(u; \partial_{\theta_i} \Phi_\theta\right).
\end{equation}

Now, in order to elaborate on some computational aspects, let us state the actual gradients of the corresponding losses. For a neural network $u=\Phi_\theta$, we compute
\begin{subequations}
\begin{align}
\label{eq: gradient FK loss estimator}
       \nabla_{\theta} \widehat{\mathcal{L}}^{(K)}_{\mathrm{FK}}(u) &=- \frac{2}{K}\sum_{k = 1}^{K} \widehat{\Delta}_u^{(k)} \nabla_\theta u,\\
\label{eq: gradient BSDE loss estimator} \nabla_{\theta} \widehat{\mathcal{L}}^{(K)}_{\mathrm{BSDE}}(u) &= -\frac{2}{K}\sum_{k = 1}^{K} e^{(k)} \left( \nabla_\theta u + \nabla_\theta \widehat{S}_u^{(k)} \right),
\end{align}
\end{subequations}
where $e^{(k)} \coloneqq  \widehat{\Delta}_u^{(k)} - \widehat{S}_u^{(k)}$ is the sample-wise error.

From a computational point of view, note that the derivative of $\widehat{\mathcal{L}}^{(K)}_{\mathrm{FK}}$ as displayed in \eqref{eq: gradient FK loss estimator} only relies on a single gradient computation of the approximating function $u$ (w.r.t.\@ the model parameters $\theta$), whereas the derivative of $\widehat{\mathcal{L}}^{(K)}_{\mathrm{BSDE}}$ as displayed in \eqref{eq: gradient BSDE loss estimator} needs evaluations of $\nabla_\theta s_j^{(k)}$, see~\eqref{eq: estimator discr} and~\eqref{eq: stoch integral discr incr}. This might be costly for automatic differentiation (autodiff) tools and we will suggest computational speed-ups and memory-efficient versions in the next section.

Further, note that due to the term $\widehat{S}^{(k)}_u$ the loss $\widehat{\mathcal{L}}^{(K)}_{\mathrm{BSDE}}$ always needs to rely on discretized SDEs, whereas $\widehat{\mathcal{L}}^{(K)}_{\mathrm{FK}}$ can be computed without SDE simulations whenever explicit solutions of the stochastic processes are available (which, admittedly, is often not the case in practice -- see, however, the examples in \Cref{sec: heat equation,sec: black-scholes}).

\subsection{Further Computational Adjustments}
\label{sec: computational}
For the loss $\mathcal{L}_\mathrm{BSDE}$ as defined in \eqref{eq: BSDE loss} we can make further adjustments that will turn out to be relevant from a computational point of view. In particular, let us suggest two alternative versions that are based on the fact that the function $u$ in \eqref{eq: BSDE loss} appears in two instances, which can be treated differently. One idea is to differentiate only with respect to one of its appearances. To this end, let us define the loss
\begin{equation}
    \mathcal{L}_\mathrm{BSDE}^\mathrm{detach}(u, w) := \E\left[\left(\Delta_u -  S_w \right)^2 \right],
\end{equation}
now depending on two arguments. We then compute derivatives according to
\begin{equation} 
\label{eq:detach}
\frac{\delta}{\delta u} \Big|_{w = u}\mathcal{L}_\mathrm{BSDE}^\mathrm{detach}(u, w),
\end{equation}
where in automatic differentiation tools setting $w = u$ after computing the derivative is usually obtained by detaching $w$ from the computational graph\footnote{In PyTorch and TensorFlow this can be achieved by the \texttt{detach} and \texttt{stop\_gradient} operations.}. One can readily check that $\mathcal{L}_\mathrm{BSDE}^\mathrm{detach}$ is a valid loss for problem \eqref{eq:PDE} that keeps the robustness property stated in \Cref{prop: variance of losses,prop:robust} intact. At the same time we expect a great reduction of the computational effort since the second-order derivatives of the form $\nabla_\theta \partial_{x_i} u$, appearing in $\nabla_\theta \widehat{S}_u^{(k)}$, do not need to be computed anymore -- in particular note that multiple such derivatives would need to be evaluated in actual implementations due to the discretization of the integral, as stated in \eqref{eq: bsde estimator discr}. We refer to \Cref{sec: examples} where we can indeed demonstrate substantial computational improvements in several settings.

Another version of $\mathcal{L}_\mathrm{BSDE}$ relies on explicitly modelling the gradient of the approximating function $u$ by an extra function, see also~\citet{zhou2021actor}. We can therefore define the loss
\begin{equation}
\label{eq: BSDE loss extra model}
    \mathcal{L}_\text{BSDE}^\text{grad}(u, r) := \E\left[\left(\Delta_u -  \widetilde{S}_r \right)^2 \right],
\end{equation}
where now
\begin{equation}
    \widetilde{S}_r = \int_\tau^T \big(\sigma^\top r \big)(X_s, s) \cdot \mathrm dW_s.
\end{equation}
The additional function $r : \R^d \times [0, T] \to \R^d$ can be modelled with a separate neural network. The loss \eqref{eq: BSDE loss extra model} is then minimized with respect to the parameters of the functions $u$ and $r$ simultaneously, which again avoids the computation of second-order derivatives.

Finally, let us introduce another viable option for decreasing computational resources, which is essentially a more memory-efficient way of computing the derivative formula stated in \eqref{eq: gradient BSDE loss estimator}. Usually, one performs a forward pass to compute $\widehat{\mathcal{L}}^{(K)}_{\mathrm{BSDE}}(u)$ and then 
uses automatic differentiation to obtain the gradient $\nabla_{\theta} \widehat{\mathcal{L}}^{(K)}_{\mathrm{BSDE}}(u)$. However, this requires to keep track of all the discretization steps for the stochastic integral $(s_j^{(k)})_{j=1}^{J^{(k)}}$, as they all depend on $\theta$. This leads to a GPU memory consumption which scales linearly in the number of steps $J^{(k)}$ and, already for small $K$, exceeds common capacities, see \Cref{fig: scaling bs,fig: gpu mem} in the appendix.

To circumvent this issue, we observe that the derivative in~\eqref{eq: gradient BSDE loss estimator} can be decomposed into 
\begin{equation}
   G_1 \coloneqq -\frac{2}{K}\sum_{k = 1}^{K} e^{(k)} \nabla_\theta u,
\end{equation}
which equals the derivative of the loss $\mathcal{L}_\mathrm{BSDE}^\mathrm{detach}$, and
\begin{equation}
  \sum_{j=1}^{J} G_{2,j} \coloneqq  \sum_{j=1}^{J} -\frac{2}{K}\sum_{k = 1}^{K}   e^{(k)}\nabla_\theta s_j^{(k)},
\end{equation}
where $J= \max_{k=1}^K J^{(k)}$ and $s^{(k)}_{j}=0$ for $j > J^{(k)}$.
Therefore, we can first compute and cache the errors $(e^{(k)})_{k=1}^K$ and the gradients $G_1$. Then, we repeat the same simulation, i.e., use the same realization of $(\xi^{(k)},\tau^{(k)},W^{(k)})$, to compute and accumulate the gradients $(G_{2,j})_{j=1}^{J}$ on-the-fly. In doing so, automatic differentiation only needs to track a single discretization step $s_j^{(k)}$. This keeps the memory footprint independent of the number of steps $J^{(k)}$ or, equivalently, the step-size $t^{(k)}_{j+1} - t^{(k)}_j$ -- however, at the cost of an increased computational time\footnote{For most practical use-cases one is interested in obtaining the most accurate solution and there is only little restriction on the time for training, which needs to be done once per PDE. However, for time-critical applications, we present the performance w.r.t.\@ the training time in Figures~\ref{fig: heat time} and~\ref{fig: heat time scaling} in the appendix.}. Of course, the same trick can also be applied to the loss $\mathcal{L}_\mathrm{BSDE}^\mathrm{grad}$. Even though from a mathematical perspective the losses are still the same as $\mathcal{L}_\mathrm{BSDE}$ and $\mathcal{L}_\mathrm{BSDE}^\mathrm{grad}$, we give them the new names
\begin{equation}
     \mathcal{L}_\mathrm{BSDE,\,eff} \qquad \text{and} \qquad \mathcal{L}_\mathrm{BSDE,\,eff}^\mathrm{grad}
\end{equation}
in order to relate to their significant computational improvements, which we will illustrate in the next section.

\section{Numerical Examples}
\label{sec: examples}

In this section we aim to illustrate our theoretical results on three high-dimensional PDEs\footnote{The associated code can be found at \url{https://github.com/juliusberner/robust_kolmogorov}.}. Specifically, we consider a heat equation with paraboloid terminal condition, a HJB equation arising in molecular dynamics, and a Black-Scholes model with correlated noise. The respective reference solutions are given as a closed-form expression, as a tensor-product of one-dimensional approximations obtained by finite-difference methods, and as a Monte-Carlo approximation using the Feynman-Kac formula, see Appendix~\ref{app: feynman-kac}.

In all of our numerical experiments we use comparable setups which are detailed in \Cref{app: computational details}.
We emphasize that, instead of using a train/val/test split on a given finite data set, our setting allows us to simulate new i.i.d.\@ samples $(\xi^{(k)},\tau^{(k)}, W^{(k)})$ for each batch on demand during training. To evaluate our methods, we compute the MSE between the neural network approximation and the reference solution at points $(x,t)$ which are sampled independently of the training data, i.e., on previously unseen points, see \Cref{app: computational details}.

We will show that all considered losses are viable in the sense that they yield appropriate approximations of solutions to the associated PDEs. At the same time we will see substantial performance differences, which can be attributed to the different variance properties outlined before. 

\subsection{Heat Equation} 
\label{sec: heat equation}
Let us first consider a version of the heat equation with paraboloid terminal condition by setting $b(x,t)=0$, $\sigma(x,t)=\bar{\sigma}\in \R^{d\times d}$, and $g(x)=\|x\|^2$. This allows for explicit solutions given by
\begin{equation}
    V(x,t) = \|x\|^2+\operatorname{Trace}(\bar{\sigma} \bar{\sigma}^\top)(T-t)
\end{equation}
and
\begin{equation}
    X_s = \xi + \bar{\sigma} (W_s - W_\tau),
\end{equation}
see also~\citet{beck2021solving}.
We analyze the special case $\bar{\sigma}=\nu\,\operatorname{Id}$ with $\nu \in \R$, where we have that
\begin{equation}
\label{eq:int_heat}
     \int_{\tau}^T \bar{\sigma}^\top \nabla_x V(X_s, s) \cdot \mathrm dW_s = 2 \,\nu \int_{\tau}^T X_s \cdot \mathrm dW_s.
\end{equation}
Using the fact that $\int_{0}^{\mathcal{T}} B_s \cdot \mathrm dB_s = \frac{1}{2}\big(\|B_\mathcal{T}\|^2-d\mathcal{T}\big)$, where $\mathcal{T}=T-\tau$ and $B_s = W_{\tau+s}-W_\tau$, the stochastic integral in~\eqref{eq:int_heat} evaluates to the following closed-form formula
\begin{equation}
    2 \nu \xi \cdot(W_T - W_\tau) + \nu^2\big(\|W_T-W_\tau\|^2 -d(T-\tau)\big).
\end{equation}
This expression shows how the quantities in \Cref{prop: variance of losses,prop:robust} may behave. For a fair comparison, however, we do not make use of this explicit representation.

For our experiments, we consider $d=50$, $T=1$, $\nu=0.5$, $\xi \sim \mathrm{Unif}([-0.5,0.5]^d)$, and $\tau\sim \mathrm{Unif}([0,1])$.
For a fixed batch size all of our proposed methods significantly outperform the standard approach of minimizing $\mathcal{L}_{\mathrm{FK}}$, see \Cref{fig: loss std first page}. For larger batch sizes the losses $\mathcal{L}_{\mathrm{BSDE}}$ and $\mathcal{L}^{\mathrm{grad}}_{\mathrm{BSDE}}$ run out of our memory limit of $8$ GiB and the performance of $\mathcal{L}_{\mathrm{FK}}$ improves. However, when using the efficient versions or $\mathcal{L}^{\mathrm{detach}}_{\mathrm{BSDE}}$, we still outperform the baseline by a substantial margin. As suggested by our theory, the performance gains are based on the variance reducing effect of the stochastic integral, see Figures \ref{fig: loss std first page} and \ref{fig: grad std heat}. While the gradient $\nabla_x V$ can be best approximated by $\mathcal{L}^{\mathrm{grad}}_{\mathrm{BSDE,\,eff}}$ (see also \Cref{fig: heat plot} in the appendix), which reduces the variance according to \Cref{prop:robust_around}, the loss $\mathcal{L}_{\mathrm{BSDE,\,eff}}$ still performs better in approximating the solution $V$. This can be motivated by the fact that in case of $\mathcal{L}^{\mathrm{grad}}_{\mathrm{BSDE,\,eff}}$ one network learns the gradient and another one independently learns the solution. In contrast, in the case of $\mathcal{L}_{\mathrm{BSDE,\,eff}}$, a single network is trained to simultaneously approximate $V$ and its gradient $\nabla_x V$. Note that these observations are consistent across various diffusivities $\nu$, see Figures \ref{fig: diff scaling} and \ref{fig: heat diff scaling}.

\begin{figure}
  \centering
  \includegraphics[width=1.0\linewidth]{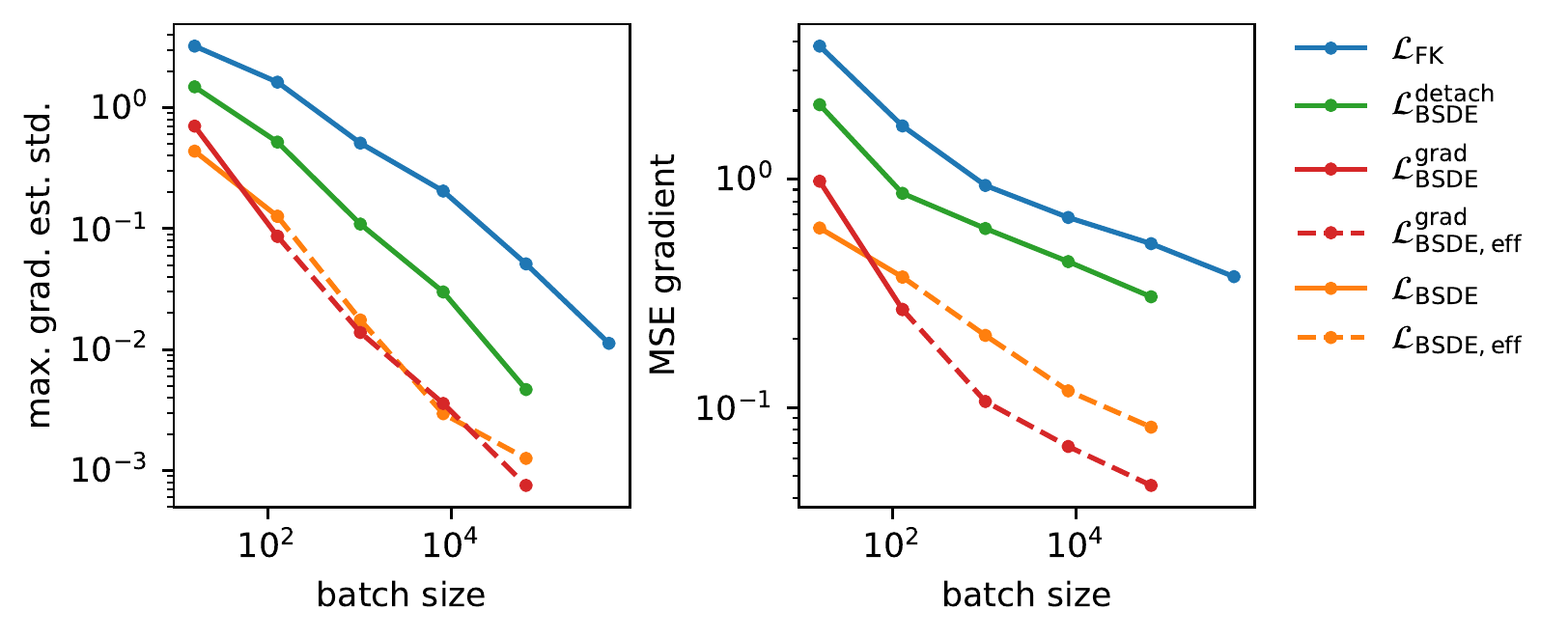}
  \vspace{-0.7cm}
  \caption{Standard deviation of the
  gradient estimators (maximum over directions $\phi = \partial_{ \theta_i}\Phi_{\theta}$) and MSE of the gradients (on evaluation data)
  in the setup of \Cref{fig: loss std first page}, see also \Cref{app: computational details}. In line with \Cref{prop:robust_around} and Remark~\ref{rem: approx grad}, the standard deviations of the loss and gradient estimators are smaller for the BSDE-based losses, especially when the gradient $\nabla_x V$ is well approximated.}
  \label{fig: grad std heat}
\end{figure}

\subsection{Black--Scholes Model with Correlated Noise}
\label{sec: black-scholes}

As a second example we consider a version of the celebrated Black--Scholes model from financial engineering \citep{black1973pricing}, given by
\begin{equation}
    \sigma(x,t) =  \operatorname{diag}(\beta_1  x_1, \dots, \beta_d x_d) \bar{\sigma}, \quad b(x,t) = \bar{b}x,
\end{equation}
with $\beta \in \R^d, \bar{\sigma}\in\R^{d\times d}, \bar{b} \in \R$, and
\begin{equation}
    g(x) = \max\left\{0, \kappa-\min_{i=1}^d x_i\right\}.
\end{equation}
It models a rainbow European put option, giving its holder the right to sell the minimum of underlying assets at the strike price $\kappa\in\R_{>0}$ at time $T$.

For the experiments we choose $d=50$, $T=1$, $\bar{b}=-0.05$, $\kappa=5.5$,
$\beta = (0.1+i/(2d))_{i=1}^d$,
\begin{equation}
    \xi\sim \mathrm{Unif}([4.5,5.5]^d), \quad \tau \sim \mathrm{Unif}([0,1]),
\end{equation}
and $\bar{\sigma}$ to be the lower triangular matrix arising from the Cholesky decomposition
$\bar{\sigma}\bar{\sigma}^\top = Q$ with \begin{equation}
    Q_{i,j}=0.5 (1 + \delta_{i,j}), \quad i,j\in\{1,\dots,d\},
\end{equation}
where $\delta_{i,j}$ denotes the Kronecker delta. There does not exist a closed form solution $V$, however, one can use Monte Carlo sampling and the representation
\begin{equation}
    (X_s)_i = \xi_i e^{ \big(\bar{b} - \frac{\| \beta_i\sigma_i \|^2}{2} \big) (s-\tau) + \beta_i\sigma_i \cdot (W_{s}-W_{\tau})},
\end{equation}
where $\sigma_i$ is the $i$-th row of $\bar{\sigma}$,
in order to evaluate the solution pointwise, see~\citet{beck2021solving}. Note that, in case of the loss $\mathcal{L}_\mathrm{FK}$, such a closed-form expression could be used instead of the Euler-Maruyama scheme in order to speed up training. However, for most settings such a representation is not available and, for the sake of a fair comparison, we did not make use of it during training.

Note that for this example the coefficient functions $\sigma$ and $b$ depend on $x$, which makes the discretization of the SDE more delicate. The loss $\mathcal{L}^{\mathrm{detach}}_{\mathrm{BSDE}}$ can be sensitive to initial performance and performs suboptimally in this example, see Remark~\ref{rem: approx grad} and \Cref{fig: bs detach}. The other losses again underline our theoretical results and, in particular, substantially outperform the loss $\mathcal{L}_\mathrm{FK}$, see \Cref{fig: scaling bs}. As suggested by our results, a larger batch size generally improves the robustness. However, the efficient versions of $\mathcal{L}_{\mathrm{BSDE}}$ and $\mathcal{L}^{\mathrm{grad}}_{\mathrm{BSDE}}$ do not need significantly more GPU memory than $\mathcal{L}_\mathrm{FK}$ such that they can be used with the same maximal batch size.

\begin{figure}
  \centering
  \includegraphics[width=1.0\linewidth]{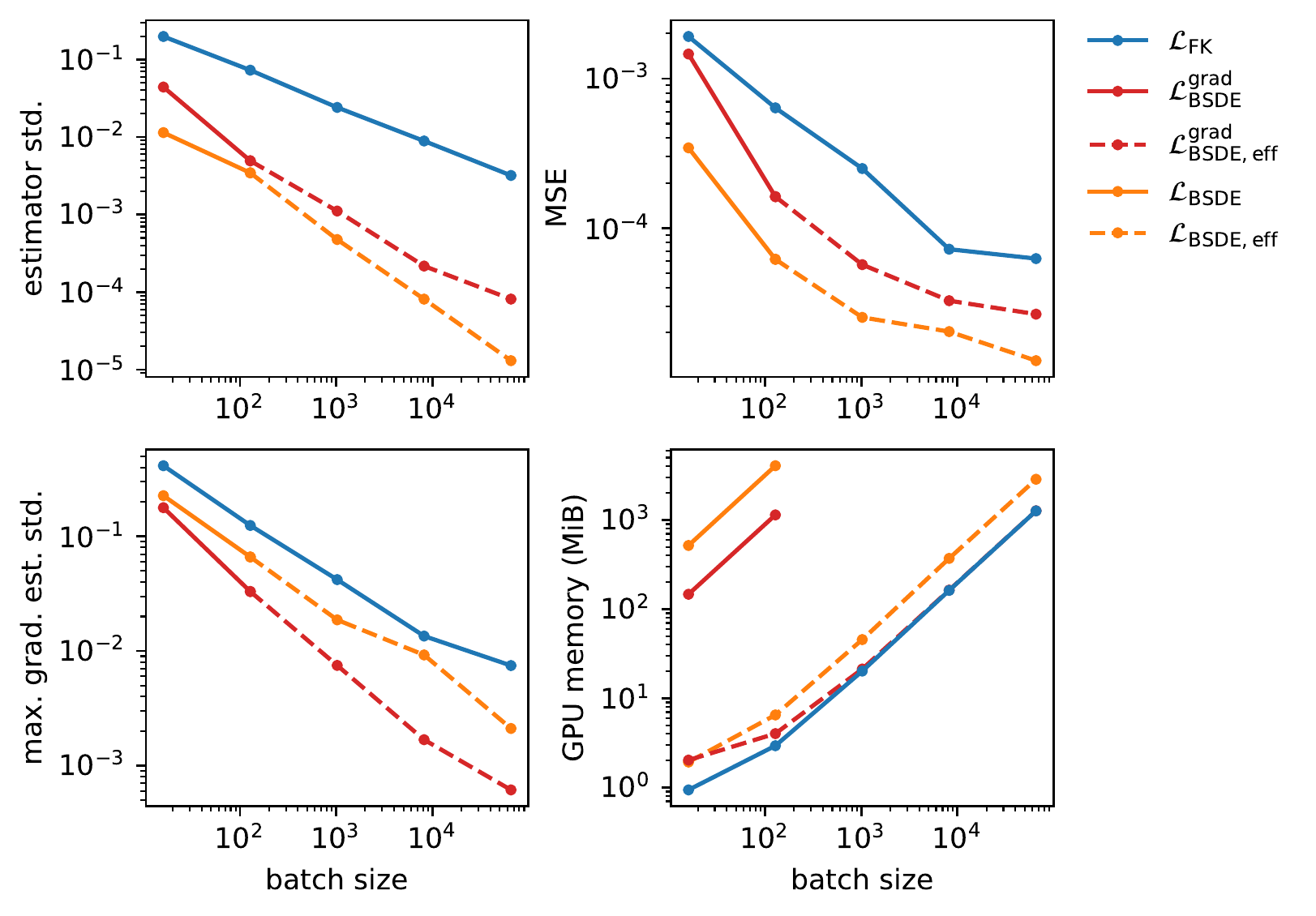}
  \vspace{-0.7cm}
  \caption{Estimator standard deviation, gradient estimator standard deviation (maximum over directions $\phi = \partial_{ \theta_i}\Phi_{\theta}$), and MSE (on evaluation data) after $30k$ gradient steps for the Black--Scholes model in \Cref{sec: black-scholes}. The last plot shows that $\mathcal{L}_{\mathrm{BSDE}}$ and $\mathcal{L}^{\mathrm{grad}}_{\mathrm{BSDE}}$ cannot be used for batch size $K=1024$ and larger as it would exceed the GPU memory limit of $8$ GiB. The same holds for the other losses at batch size $K=524288$.}
  \label{fig: scaling bs}
\end{figure}

\subsection{Hamilton--Jacobi--Bellman Equation}
\label{sec: hjb}
Finally, let us consider the nonlinear Hamilton--Jacobi--Bellman equation
\begin{align}
\label{eq: HJB PDE}
    (\partial_t + L)\widetilde{V}(x, t) + f(x, t)  = \frac{1}{2} \big\|(\sigma^\top \nabla_x \widetilde{V})(x, t)\big\|^2
\end{align}
for $(x, t) \in {\R}^d \times [0, T)$, with terminal condition
\begin{align}
    \widetilde{V}(x, T) &= \widetilde{g}(x)
\end{align}
for $x \in {\R}^d$. This prominent PDE from control theory corresponds to the controlled stochastic process
\begin{equation*}
    \mathrm dX^v_s = \left(b(X^v_s, s) + (\sigma v)(X^v_s, s) \right)\,\mathrm d s + \sigma(X^v_s, s) \,\mathrm d W_s,
\end{equation*}
where the control $v \in C(\R^d \times [0, T], \R^d)$ can be thought of as a steering force to be chosen so as to minimize a given cost function
\begin{equation}
\label{eq: control costs}
    J(v) = \E\left[\int_0^T \left(f_s + \frac{1}{2}\|v_s\|^2\right)\mathrm ds + \widetilde{g}(X_T^v)  \right],
\end{equation}
with the short-hands $f_s \coloneqq f(X_s^v, s)$ and $v_s \coloneqq v(X_s^v, s)$.
In the above, $f \in C(\R^d \times [0, T], \R)$ represents running costs (in addition to the squared costs on the control $v$) and $\widetilde{g} \in C(\R^d, \R)$ specifies the terminal costs. It turns out that one can recover the optimal control $v^*$ that minimizes the costs \eqref{eq: control costs} from the solution to the PDE in \eqref{eq: HJB PDE} by the relation $v^* = -\sigma^\top \nabla_x \widetilde{V}$.

Now, with the transformation\footnote{This transformation is also known as \emph{Hopf-Cole transformation}, see, e.g., Section 4.4.1 in \citet{evans2010}. See also \citet{hartmann2017variational} for a discussion on related applications in importance sampling of stochastic processes relevant in computational statistical physics.} 
$V = \exp(-\widetilde{V})$, we can convert the nonlinear PDE in \eqref{eq: HJB PDE} to a linear one of the form \eqref{eq:PDE} (or, to be precise, of the form \eqref{eq: general Feynman-Kac BVP}) with boundary condition $g = \exp(-\widetilde{g})$, which allows us to apply \Cref{alg: algorithm}, see \Cref{lem: Linearization of HJB equation} for the details.

For our experiments, we consider a problem that is prominent in molecular dynamics and has been suggested in \citet{nusken2021solving}. We define the drift to be $b = -\nabla_x \Psi$, with $ \Psi(x) = \kappa \sum_{i=1}^d (x_i^2 - 1)^2$
being a potential of double well type. In applications, the potential represents the energy associated to a system of atoms (i.e. molecule) and one is often interested in how molecular configurations change over time. This is then related to transition paths between metastable states of the stochastic dynamics, whose sampling poses formidable computational challenges, see e.g. \citet{hartmann2021nonasymptotic}.

For an example, let us define a terminal value that is supported in one of the $2^d$ minima of $\Psi$, namely $g(x) = \exp(-\eta \sum_{i=1}^d (x_i - 1)^2 )$ and set $f = 0$. We choose $\eta = 0.04$ and $\kappa = 0.1$ in dimension $d = 10$ as well as $\sigma = \operatorname{Id}$. The left panel of \Cref{fig: HJB double well} shows the original potential $\Psi$ as well as a tilted version that can be derived from the solution of the HJB equation \eqref{eq: HJB PDE} via $\Psi^* = \Psi + \sigma \sigma^\top \widetilde{V}$, computed with a finite difference reference method as well as with our algorithm. We can see that both solutions agree for all of our considered methods. However, note that for the optimal control functions the BSDE-based losses are again superior to $\mathcal{L}_{\mathrm{FK}}$, as displayed in the right panel. This is in line with the more accurate approximation of the gradient, which is demonstrated by a performance comparison in \Cref{fig: hjb scaling} in the appendix.

\begin{figure}
  \centering
  \includegraphics[width=1.0\linewidth]{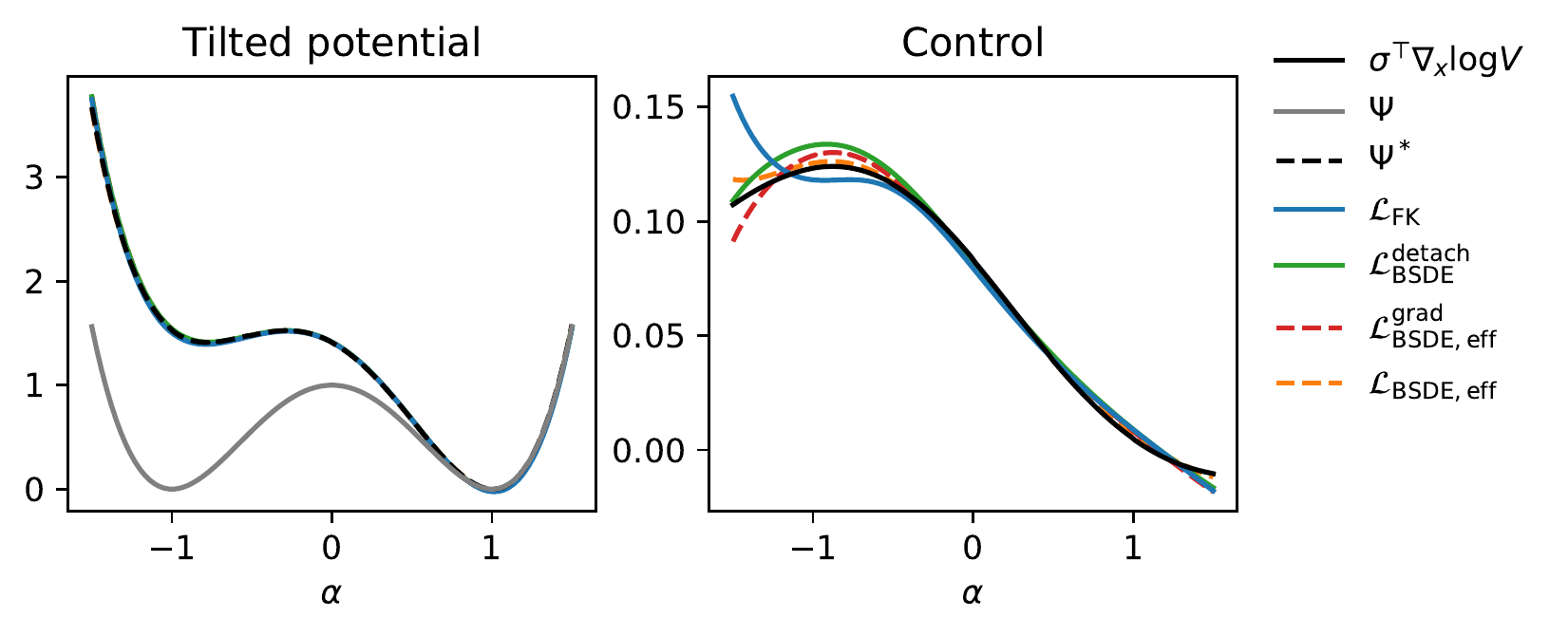}
  \vspace{-0.7cm}
  \caption{For a HJB equation with double well potential we display the tilted version of the potential $\Psi^*$ evaluated at $x = (\alpha, \dots, \alpha)^\top$ and $t=0.75$, which is approximated well by all our methods (after $30k$ gradient steps with batch size $1024$). However, the approximation of the first component of the control $\sigma^\top\nabla_x \log V$ depicted in the right panel is better for BSDE-based losses.}
  \label{fig: HJB double well}
  \vspace{-0.25cm}
\end{figure}

\section{Conclusion}
The most prominent method for solving general high-dimensional linear PDEs of Kolmogorov type is based on minimizing a variational formulation inspired by the Feynman-Kac formula over a class of neural networks. Building upon results on BSDEs and control variates for Monte Carlo estimators, we suggest an alternative formulation and prove that the corresponding loss estimator enjoys lower variance. Lending to the notion of Gateaux derivatives, we can further show that this leads to gradient estimators with lower variance, which is crucial for gradient-based optimization. Importantly, we also develop a novel, more efficient estimator that is adapted to current deep learning frameworks and show that resulting methods yield significant performance gains. It is up to future research to exploit such efficient estimators for boundary value problems, see Appendix~\ref{app: feynman-kac}, and non-linear PDEs, which can naturally be approached with BSDE-based losses, see Appendix~\ref{app: BSDEs}.

In summary, we suspect that the techniques used in this paper can be transferred to a number of related deep learning methods that rely on Monte Carlo sampling. Providing important theoretical guarantees and performance improvements, our work is a fundamental step in developing robust and reliable algorithms for the solution of high-dimensional PDEs. 

\section*{Acknowledgements}
The research of Lorenz Richter has been funded by Deutsche Forschungsgemeinschaft (DFG) through the grant CRC 1114 \enquote{Scaling Cascades in Complex Systems} (project A05, project number 235221301). The research of Julius Berner was supported by the Austrian Science Fund (FWF) under grant I3403-N32. The computational results have been achieved in part using the Vienna Scientific Cluster (VSC).

\bibliography{references}
\bibliographystyle{icml2022}

\clearpage
\newpage

\appendix
\section{Notation}
\label{app: notation}
We write $C^{2, 1}(\R^d \times [0, T], \R)$ for the space of functions $(x, t) \mapsto f(x,t)$ which are twice continuously differentiable in the spatial coordinate $x\in\R^d$ and once continuously differentiable in the time coordinate $t\in[0,T]$. We further denote the partial derivatives (w.r.t.\@ to the $i$-th spatial coordinate $x_i$ and time coordinate $t$) and the gradient of $f$ (w.r.t.\@ to the spatial coordinate $x$) by $\partial_{x_i} f$, $\partial_t f$, and $\nabla_x f = ( \partial_{x_1} f , \dots, \partial_{x_d} f)^\top $, respectively.
We say that functions $f\in C(\R^d \times [0, T], \R)$ and $g\in C(\R^d, \R)$ are \emph{at most polynomially growing} if there exist constants $c,\lambda \in\R_{>0}$ such that for every $(x,t) \in \mathbb{R}^d \times [0,T]$ it holds that
\begin{equation}
    |g(x)| \le c(1+ \|x\|^{\lambda}), \quad |f(x,t)| \le c(1+ \|x\|^{\lambda}).
\end{equation}
Finally, we write $\|\cdot \|$ for the Euclidean norm and $\|\cdot\|_F$ for the Frobenius norm.
\section{Assumptions and Proofs}
\label{app: proofs}
We usually operate under the following assumption which guarantees a unique strong solution to the SDE in~\eqref{eq:SDE} as well as the existence of a unique strong solution $V\in C^{2,1}(\R^d \times [0, T], \R)$ to the PDE in~\eqref{eq:PDE} which can be expressed as $V(x, t) = \E[g(X_T) | X_t = x]$ as, for instance, stated in \eqref{eq: FK contitional expectation}, see~\citet[Theorem 10.6]{baldi2017stochastic}, \citet[Remark 5.7.8]{karatzas1998brownian},~\citet[Theorem 5.2.1]{oksendal2003stochastic}, and~\citet[Theorem 2.1]{pavliotis2014stochastic}.
\begin{assumption}[Conditions on $b$, $\sigma$, and $g$]
\label{as: conditions on b sigma g}
We assume that $g \in C^{2}(\R^d, \R)$ is an at most polynomially growing function. Further, we assume that the coefficient functions $\sigma \in C(\R^d \times [0, T], \R^{d\times d})$ and $b \in C(\R^d \times [0, T], \R^d)$ satisfy the following properties:
There exist constants $c_1, c_2, c_3, c_4 \in\R_{>0}$ such that
\begin{align*}
&\| b(x,t) \| + \|\sigma\sigma^\top(x,t)\|_F \le c_1, \ \  \text{(boundedness)} \\
&\|\sigma(x,t) - \sigma(y,t)\|_F \le c_2\|x-y\|, \ \ \text{(Lipschitz continuity)} \\
&\|b(x,t) - b(y,t)\| \le c_3\|x-y\|, \ \ \text{(Lipschitz continuity)} \\
&\eta \cdot (\sigma\sigma^\top)(x,t) \eta \ge c_4 \| \eta \|^2, \ \ \text{(uniform ellipticity)}
\end{align*}
for all $x, y, \eta \in \mathbb{R}^d$ and $t \in [0,T]$.
\end{assumption}
Throughout the proofs, we denote by $\Delta^{(k)}_u$ and $S_u^{(k)}$ the quantities in~\eqref{eq: definition Delta_u} and \eqref{eq: definition S_u} evaluated at the $k$-th sample. This corresponds to replacing $(\xi, \tau, W)$ in~\eqref{eq:SDE},~\eqref{eq: definition Delta_u}, and~\eqref{eq: definition S_u} by the $k$-th sample $(\xi^{(k)},\tau^{(k)}, W^{(k)})$.

\begin{proof}[Proof of \Cref{prop:opt}]
We will use the fact that $\Delta_V = S_V$ almost surely, see \Cref{lem: zero variance}. This allows us to rewrite 
\begin{align*}
    \mathcal{L}_\mathrm{BSDE}(u) &= \E\left[\left(\Delta_u + V(\xi,\tau) - V(\xi,\tau) - S_u \right)^2 \right] \\ 
    &=\E\left[\left(V(\xi,\tau) - u(\xi,\tau) + S_V - S_u \right)^2 \right] \\
    &=\E\left[\left(V(\xi,\tau) - u(\xi,\tau) \right)^2 \right] + \E\left[S_{V -u}^2 \right].
\end{align*}
In the last step we used the tower property of the conditional expectation and the vanishing expectation of the stochastic integrals $S_V$ and $S_u$ to show that the cross-term has zero expectation, namely
\begin{align*}
   \E&\left[\left(V(\xi,\tau)- u(\xi,\tau) \right) (S_V - S_u)\right] \\ 
    &=\E\left[\left(V(\xi,\tau)- u(\xi,\tau) \right)\E\left[ S_V - S_u | (\xi,\tau)\right]\right]=0.
\end{align*}
For the loss $\mathcal{L}_\mathrm{FK}$ we can similarly write 
\begin{subequations}
\begin{align}
    \mathcal{L}_\mathrm{FK}(u) &= \E\left[\left(V(\xi,\tau)- u(\xi,\tau) + \Delta_V \right)^2 \right] \\ 
    &=\E\left[\left(V(\xi,\tau) - u(\xi,\tau) \right)^2 \right] + \E\left[ \Delta^2_V \right],
\end{align}
\end{subequations}
where the cross-term has again zero expectation. More precisely, the tower property of the conditional expectation and \Cref{lem: zero variance} imply that
\begin{subequations}
\begin{align}
  \E&\left[\left(V(\xi,\tau)- u(\xi,\tau) \right) \Delta_V \right] \\ 
    &=\E\left[\left(V(\xi,\tau)- u(\xi,\tau) \right)\E\left[ \Delta_V | (\xi,\tau)\right]\right]
    \\ 
    &=\E\left[\left(V(\xi,\tau)- u(\xi,\tau) \right)\E\left[ S_V | (\xi,\tau)\right]\right]=0.
\end{align}
\end{subequations}
This proves Proposition~\ref{prop:opt}.
\end{proof}

\begin{proof}[Proof of \Cref{prop: variance of losses}]
Let us recall \Cref{lem: zero variance} which states that $\Delta_V = S_V$ almost surely. With the mutual independence of $(\Delta^{(k)}_u)_{k=1}^K$ and $(S_u^{(k)})_{k=1}^K$ we now readily obtain that
\begin{align}
    \E\left[ \mathcal{L}^{(K)}_{\mathrm{BSDE}}(V) \right]  = \E\left[\left(\Delta_V - S_V\right)^2\right]=0,
\end{align}
and
\begin{align}
   \V\left[ \mathcal{L}^{(K)}_{\mathrm{BSDE}}(V) \right] = \frac{1}{K} \V\left[\left(\Delta_V - S_V\right)^2\right]=0.
\end{align}
For the loss $\mathcal{L}_{\mathrm{FK}}$, it holds that
\begin{align}
    \E\left[ \mathcal{L}^{(K)}_{\mathrm{FK}}(V) \right] = \E\left[\Delta_V^2\right] = \V\left[S_V\right]
\end{align}
and
\begin{align}
\label{eq: var_fk}
    \V\left[ \mathcal{L}^{(K)}_{\mathrm{FK}}(V) \right] = \frac{1}{K} \V\left[\Delta_V^2\right]=\frac{1}{K} \V\left[S_V^2\right],
\end{align}
which proves the claim.
\end{proof}

\begin{proof}[Proof of Proposition~\ref{prop:robust}]
(i) It holds that
\begin{subequations}
\begin{align}
    \frac{\delta}{\delta u} \mathcal{L}_\mathrm{FK}^{(K)}(u; \phi) &= \frac{\mathrm d}{\mathrm d \varepsilon}\Big|_{\varepsilon=0} \mathcal{L}^{(K)}_\mathrm{FK}(u+ \varepsilon \phi) \\
    &=\frac{\mathrm d}{\mathrm d \varepsilon}\Big|_{\varepsilon=0}\frac{1}{K}\sum_{k=1}^K\left(\Delta^{(k)}_{u + \varepsilon \phi} \right)^2 \\
    &= - \frac{2}{K} \sum_{k=1}^K\Delta^{(k)}_u  \phi(\xi^{(k)}, \tau^{(k)}).
\end{align}
\end{subequations}
Using the mutual independence of our samples this implies that
\begin{equation}
\label{eq:var_fk}
    \V\left[\frac{\delta}{\delta u}\mathcal{L}^{(K)}_\mathrm{FK}(u; \phi)\right] = \frac{4}{K} \V\left[\Delta_u \phi(\xi,\tau)   \right].
\end{equation} 
Now, setting $u= V$, \Cref{lem: zero variance}
shows that the variance of the Gateaux derivative at $V$ satisfies
\begin{equation}
   \V\left[\frac{\delta}{\delta u}{\Big|}_{u= V}\mathcal{L}^{(K)}_\mathrm{FK}(u; \phi)\right] =   \frac{4}{K} \V\left[S_V \phi(\xi,\tau)   \right].
\end{equation}
By the It\^{o} isometry, this evaluates to
\begin{equation}
    \frac{4}{K}\E\left[\left(\int_\tau^T \|\sigma^\top \nabla_x V(X_s, s) \|^2\mathrm ds \right) \phi^2(\xi,\tau)   \right],
\end{equation}
which, in general, is non-zero for an arbitrary $\phi \in \mathcal{U}$.\par\bigskip

(ii) By definition, we have that
\begin{equation}
    \mathcal{L}^{(K)}_\mathrm{BSDE}(u+ \varepsilon \phi)= \frac{1}{K}\sum_{k=1}^K\left(\Delta^{(k)}_{u+\varepsilon\phi}  - S^{(k)}_{u+\varepsilon \phi} \right)^2.
\end{equation}
The Gateaux derivative 
\begin{align}
    \frac{\delta}{\delta u} \mathcal{L}^{(K)}_\mathrm{BSDE}(u; \phi) &= \frac{\mathrm d}{\mathrm d \varepsilon}\Big|_{\varepsilon=0} \mathcal{L}^{(K)}_\mathrm{BSDE}(u+ \varepsilon \phi) 
\end{align}
thus evaluates to
\begin{equation}
\label{eq:der_bsde}
    - \frac{2}{K} \sum_{k=1}^K\left(\Delta^{(k)}_u - S^{(k)}_u \right) \left( \phi(\xi^{(k)}, \tau^{(k)}) + S^{(k)}_\phi \right).
\end{equation}
Now, setting $u=V$, \Cref{lem: zero variance} readily implies that 
\begin{equation}
    \frac{\delta}{\delta u}{\Big|}_{u= V}\mathcal{L}^{(K)}_\mathrm{BSDE}(u; \phi) = 0
\end{equation}
almost surely for all $\phi \in \mathcal{U}$. This shows that
\begin{equation}
        \V\left[\frac{\delta}{\delta u}{\Big|}_{u= V}\mathcal{L}^{(K)}_\mathrm{BSDE}(u; \phi)\right] = 0,
\end{equation}
which proves the claim. Note that the above calculation shows that the derivative of any squared loss exhibits zero variance whenever the error is vanishing almost surely.
\end{proof}

\begin{proof}[Proof of Proposition~\ref{prop:robust_around}]
Let us define
\begin{equation}
    \delta(x, t) = u(x, t) - V(x, t).
\end{equation}
Analogously to~\eqref{eq:der_bsde} we can compute
\begin{align}
\begin{split}
    \V&\left[\frac{\delta}{\delta u}{\Big|}_{u= V+\delta}\mathcal{L}^{(K)}_\mathrm{BSDE}(u; \phi)\right] \\
    &\qquad=\frac{4}{K} \V \left[ \left(\delta(\xi,\tau) + S_{\delta} \right)\left(\phi(\xi,\tau) + S_{\phi} \right)\right].
  \end{split}
\end{align}
Now, properties of the variance as well as the Cauchy-Schwarz inequality yield
\begin{align*}
    \V& \left[ \left(\delta(\xi,\tau) + S_{\delta} \right)\left(\phi(\xi,\tau) + S_{\phi} \right)\right] \\
    &\qquad \le  \E\left[\left(\delta(\xi,\tau) + S_{\delta} \right)^2\left(\phi(\xi,\tau) + S_{\phi} \right)^2\right] \\
    &\qquad \le  \E\left[\left(\delta(\xi,\tau) + S_{\delta} \right)^4\right]^{\frac{1}{2}}\E\left[\left(\phi(\xi,\tau) + S_{\phi} \right)^4\right]^{\frac{1}{2}}.
\end{align*}
Each of the above factors can be bounded by using the Burkholder-Davis-Gundy inequality \citep[Section 4.6]{da2014stochastic}, Hölder's inequality, and the consistency of the Euclidean norm, i.e.
\begin{align*}
    &\E\left[\left(\delta(\xi,\tau) + S_{\delta} \right)^4\right] \le  8\left( \E\left[\delta(\xi,\tau)^4 \right]+ \E\left[S^4_{\delta} \right]\right) \\
    &\le  8\Big( \varepsilon^4 + 36T\int_0^T \! \E\left[ \| (\sigma^\top \nabla_x\delta )(X_s,s) \|^4 \right] \mathrm d s\Big) \\
   &\le  8\varepsilon^4\Big( 1 + 36T\int_0^T \! \E\left[ \| \sigma(X_s, s) \|_F^4(1+ \|X_s\|^\gamma)^4 \right] \mathrm d s\Big).
\end{align*}
Defining
\begin{equation*}
    C\coloneqq 32\Big( 1 + 36T\int_0^T \! \E\left[\| \sigma(X_s, s) \|_F^4(1+ \|X_s\|^\gamma)^4 \right] \mathrm d s\Big),
\end{equation*}
we thus showed that
\begin{align}
    \V&\left[\frac{\delta}{\delta u}{\Big|}_{u= V+\delta}\mathcal{L}^{(K)}_\mathrm{BSDE}(u; \phi)\right] \le \frac{C\varepsilon^2 \kappa^2}{K},
\end{align}
which proves the claim.
\end{proof}

\section{Feynman-Kac Theorem and More General Linear PDEs}
\label{app: feynman-kac}

As outlined in \Cref{sec: main section} and for instance stated in \eqref{eq:fk} as well as \eqref{eq: FK contitional expectation}, the Feynman-Kac formula brings a stochastic representation of the linear PDE in \eqref{eq:PDE} via
\begin{subequations}
\begin{align}
    V(x, t) &= \E[g(X_T) | X_t = x] \\
    &=\E[g(X_T) | (\xi, \tau)=(x,t)] .
\end{align}
\end{subequations}

Note that this can be shown using the identity in~\eqref{eq:ito}. More specifically, the stochastic integral $S_V$ has a vanishing expectation conditioned on $(\xi,\tau)$ and we obtain that
\begin{equation}
\label{eq: Feynman-Kac equation compact notation}
    \E[\Delta_V - S_V |(\xi,\tau)] = \E[ \Delta_V|(\xi,\tau)] = 0.
\end{equation}

Let us make this observation precise and at the same time state a slightly more general version of the Feynman-Kac theorem.

\begin{theorem}[Feynman-Kac formula]
\label{thm: general Feynman-Kac}
Let $g \in C^2(\R^d,\R)$, $k\in C(\R^d \times [0, T], \R)$, and $V \in C^{2, 1}(\R^d \times [0, T], \R)$ be at most polynomially growing functions. Further, let $f \in C(\R^d \times [0, T], \R_{\ge 0})$ and assume that $V$ solves the parabolic problem
\begin{align}
\label{eq: general Feynman-Kac BVP}
    (\partial_t + L - f(x, t)) V(x, t) + k(x, t) &= 0
\end{align}
on $(x, t) \in \R^d \times [0, T)$ with terminal condition
\begin{align}
\label{eq: Feynman-Kac terminal condition}
    V(x, T) &= g(x), \quad\,\,\, x \in \R^d.
\end{align}
Then
\begin{align}
\label{eq: general Feynman-Kac formula}
\begin{split}
    V(x, t) = \E\Bigg[&\int_t^T e^{-\int_t^r f(X_s, s) \mathrm ds } k(X_r, r) \mathrm dr \\
    &\,\,\, + e^{-\int_t^T f(X_s, s)\mathrm ds} g(X_T) \Bigg| X_t = x\Bigg],
\end{split}
\end{align}
where $X$ is a strong solution to \eqref{eq:SDE}. 
\end{theorem}
\begin{proof}
The proof, whose main ingredient is It\^{o}'s Lemma, can, for instance, be found in \citet[Theorem 5.7.6]{karatzas1998brownian} and \citet[Theorem 10.5]{baldi2017stochastic}.
\end{proof}
Let us recall that our loss $\mathcal{L}_\mathrm{FK}$ as defined in \eqref{eq: FK loss} readily follows from the Feynman-Kac formula -- see also \citet{beck2021solving}. Note that a crucial point in the derivation of the Feynman-Kac formula is the elimination of the stochastic integral $S_V$, as defined in \eqref{eq: definition S_u}, via its martingale property, see \eqref{eq: Feynman-Kac equation compact notation}. Ironically, as elaborated on in \Cref{sec: main section}, it turns out that this elimination is just the reason for larger variances of estimators of $\mathcal{L}_\mathrm{FK}$.

\begin{remark}[Initial value problems]
Note that time can be reversed and initial value problems (as opposed to terminal value problems as in~\Cref{thm: general Feynman-Kac}) can be formulated. As an example, we can consider the time-homogeneous case, where $b, \sigma, f$, and $k$ do not depend on time. Let $V \in C^{2, 1}(\R^d \times [0, T], \R)$ solve the parabolic PDE
\begin{align}
     (\partial_t - L + f(x)) V(x, t) - k(x) = 0
\end{align}
on $(x, t) \in \R^d \times (0, T]$ with initial condition
\begin{align}
    V(x, 0) &= g(x), \quad\,\,\, x \in \R^d.
\end{align}
The associated stochastic representation is then given by
\begin{align}
\begin{split}
V(x, t) = \E\Bigg[&\int_0^t e^{-\int_0^r f(X_s) \mathrm ds } k(X_r) \mathrm dr \\
    & +  e^{-\int_0^t f(X_s)\mathrm ds} g(X_t) \Bigg| X_0 = x\Bigg].
\end{split}
\end{align}
\end{remark}

\begin{remark}[Bounded domains]
\label{rem: Feynman-Kac on bounded domain}
We can restrict ourselves to open, bounded domains $\mathcal{D} \subset \R^d$ and consider the parabolic PDE in \eqref{eq: general Feynman-Kac BVP} on $\mathcal{D}$, adding the additional boundary condition\footnote{One can also consider boundary conditions that are different from the terminal condition in \eqref{eq: Feynman-Kac terminal condition}, see, for instance,~\citet[Theorem 10.4]{baldi2017stochastic} and~\citet[Proposition 6.1]{lelievre2016partial}.} $V(x, t) = g(x)$ for $(x, t) \in \partial \mathcal{D}\times [0,T]$. The stochastic representation then becomes
\begin{align*}
    V(x, t) = \E\Bigg[&\int_t^{\mathcal{T} \wedge T} e^{-\int_t^r f(X_s, s) \mathrm ds } k(X_r, r) \mathrm dr \\
    &\,\,\, +   e^{-\int_t^{\mathcal{T} \wedge T} f(X_s, s)\mathrm ds} g(X_{\mathcal{T} \wedge T}) \Bigg| X_t = x\Bigg],
\end{align*}
where $T \wedge \mathcal{T} \coloneqq \min\{ T, \mathcal{T}\}$ and $\mathcal{T} \coloneqq \inf \{t \ge 0: X_t \notin \mathcal{D}\}$ is the first exit time of the stochastic process from the domain $\mathcal{D}$, for which we usually assume $\mathcal{T} < \infty$ almost surely. Likewise, we can consider the elliptic boundary value problem
\begin{subequations}
\begin{align}
    (L - f(x)) V(x) + k(x) &= 0, \quad &  x \in \mathcal{D}, \\
    V(x) &= g(x), \quad & x \in \partial \mathcal{D},
\end{align}
\end{subequations}
where now the solution $V$, the coefficients $b$ and $\sigma$ in the SDE \eqref{eq:SDE}, as well as $f$ and $k$ do not depend explicitly on time anymore, yielding the stochastic representation
\begin{align}
\begin{split}
\label{eq:bvp_2}
    V(x) = \E\Bigg[&\int_0^{\mathcal{T}} e^{-\int_0^r f(X_s) \mathrm ds } k(X_r) \mathrm dr \\
    &\,\,\, + e^{-\int_0^{\mathcal{T}} f(X_s)\mathrm ds} g(X_\mathcal{T}) \Bigg| X_0 = x\Bigg],
\end{split}
\end{align}
again with $\mathcal{T} \coloneqq \inf \{t \ge 0: X_t \notin \mathcal{D}\}$, see e.g. Proposition 5.7.2 in \citet{karatzas1998brownian}.

Leveraging the above representations, our proposed methods can readily be applied to a range of elliptic and parabolic PDEs on bounded domains. Note, however, that one needs to take into account the hitting times $\mathcal{T}$, for instance, using naive stopping criteria~\cite{nusken2021interpolating} or more elaborate walk-on-the-sphere algorithms~\cite{grohs2020deep}. As this might obscure the comparisons between different loss functions, we focus on unbounded domains in our experiments.
\end{remark}

\section{Backward Stochastic Differential Equations and Semi-Linear PDEs}
\label{app: BSDEs}

Backward stochastic differential equations (BSDEs) have been studied extensively in the last three decades and we refer to \citet{pardoux1998backward}, \citet{pham2009continuous}, \citet{gobet2016monte}, and~\citet{zhang2017backward} for good introductions to the topic. Even though in this paper we only consider linear PDEs as stated in \eqref{eq:PDE} or \eqref{eq: general Feynman-Kac BVP}, BSDEs are typically associated to nonlinear (parabolic) PDEs of the type
\begin{align}
\label{eq: definition general PDE}
    (\partial_t + L) V(x, t) = h(x, t, V(x, t), (\sigma^\top \nabla_x V)(x, t)) 
\end{align}
for $(x, t) \in {\R}^d \times [0, T]$, a nonlinearity $h: \R^d \times [0, T] \times \R \times \R^d \to \R $, and differential operator $L$ defined as in~\eqref{eq:diff_operator}. The terminal value is given by
\begin{equation}
    V(x, T) = g(x),
\end{equation}
for a specified function $g \in C(\R^d, \R)$. 

BSDEs were first introduced in \citet{bismut1973conjugate} and their systematic study began with \citet{pardoux1990adapted}. Loosely speaking, they can be understood as nonlinear extensions of the Feynman-Kac formula (see \citet{pardoux1998backward} and \Cref{app: feynman-kac}), relating the nonlinear PDE in 
\eqref{eq: definition general PDE} to the stochastic process $X_s$ defined by
\begin{equation}
\label{eq: fordward SDE}
    \mathrm dX_s = b(X_s, s) \, \mathrm ds + \sigma(X_s, s) \, \mathrm d W_s, \quad X_0 = x_0.
\end{equation}
The key idea is then to define the processes
\begin{equation}
\label{eq: def Y Z}
    Y_s = V(X_s, s), \qquad Z_s = (\sigma^\top \nabla_x V)(X_s, s),
\end{equation}
as representations of the PDE solution and its gradient, respectively,
and apply It\^{o}'s lemma to obtain
\begin{equation}
\label{eq: BSDE}
    \mathrm d Y_s  = h(X_s, s, Y_s, Z_s) \, \mathrm ds +  Z_s \cdot \mathrm dW_s, 
\end{equation}
with terminal condition $Y_T  = g(X_T)$.
Noting that the processes $Y$ and $Z$ are adapted\footnote{Intuitively, this means that the processes $Y$ and $Z$ must not depend on future values of the Brownian motion $W$.} to the filtration generated by the Brownian motion $W$, they should indeed be understood as backward processes and not be confused with time-reversed processes. A convenient interpretation of the relations in \eqref{eq: def Y Z} is that solving for the processes $Y$ and $Z$
under the constraint \eqref{eq: BSDE} corresponds to determining the solution of the PDE in \eqref{eq: definition general PDE} (and its gradient) along a random grid which is provided by the stochastic process $X$ defined in \eqref{eq: fordward SDE}.

Let us note that under suitable assumptions on the coefficients $b, \sigma, h$, and $g$ one can prove existence and uniqueness of a solution to the BSDE system as defined in \eqref{eq: fordward SDE} and \eqref{eq: BSDE}, see for instance Theorem 4.3.1 in \citet{zhang2017backward}.

We further note that the standard BSDE system can be generalized to
\begin{align*}
    \mathrm dX_s^v &\! = \! \left(b(X_s^v, s) + v(X_s^v, s)\right) \mathrm ds + \sigma(X_s^v, s) \mathrm d W_s, \\
      \mathrm dY_s^v & \! = \! (h(X_s^v, s, Y_s^v, Z_s^v)  + v(X_s^v, s)\! \cdot \! Z_s^v) \mathrm ds + Z_s^v \!\cdot \mathrm d W_s,
\end{align*}
with 
\begin{equation}
    X_0^v = x, \qquad Y_T^v = g(X_T^v).
\end{equation}
In the above, $v \in C( \R^d \times [0, T], \R^d)$ is a suitable control vector field that can be understood as pushing the forward trajectories into desired regions of the state space, noting that the relations 
\begin{equation}
    Y_s^v = V(X_s^v, s), \qquad Z_s^v = (\sigma^\top \nabla_x V)(X_s^v, s),
\end{equation}
with $V \in C^{2, 1}(\R^d \times [0, T], \R)$ being the solution to the parabolic PDE in \eqref{eq: definition general PDE}, hold true independent of the choice of $v$ \cite{hartmann2019variational}. 

\section{Further Computational Details}
\label{app: computational details}

\begin{algorithm}[t!]
\begin{algorithmic}
\INPUT Neural network $\Phi$ with initial parameters $\theta^{(0)}$, optimizer method $\operatorname{step}$ for updating the parameters, maximum number of steps $M$, batch size $K$, step-size $\Delta t$
\OUTPUT parameters $\theta^{(M)}$

\FOR{$m\gets 0,\dots, M-1$}
\STATE $ (\xi^{(k)},\tau^{(k)},W^{(k)})_{k=1}^K \gets \text{sample from } (\xi,\tau,W)^{\otimes K}$ 
\STATE $(\widehat{X}^{(k)})_{k=1}^K \gets \text{simulate using the EM scheme in~\eqref{eq: Euler scheme}}$
\STATE $\mathcal{L} \gets \text{pick } \mathcal{L} \in \{ \mathcal{L}_\mathrm{FK}, \mathcal{L}_\mathrm{BSDE}\}$
\STATE $\widehat{\mathcal{L}}^{(K)}(\Phi_{\theta^{(m)}}) \gets \text{compute estimator loss as in~\eqref{eq: loss estimators}}$
\STATE $\nabla_\theta\widehat{\mathcal{L}}^{(K)}(\Phi_{\theta^{(m)}}) \gets \operatorname{autodiff}(\widehat{\mathcal{L}}^{(K)}(\Phi_{\theta^{(m)}}))$
\STATE $\theta^{(m + 1)} \gets \operatorname{step}\left( \theta^{(m)}, \nabla_\theta \widehat{\mathcal{L}}^{(K)}(\Phi_{\theta^{(m)}})\right)$ 
\ENDFOR
\end{algorithmic}
\caption{Solving the PDE in~\eqref{eq:PDE} via deep learning}
\label{alg: algorithm}
\end{algorithm}

For convenience, we first summarize our considered method in~\Cref{alg: algorithm}. In numerical simulations we need to discretize the stochastic process $X$ as defined in \eqref{eq:SDE} on a time grid $\tau = t_1 < \dots < t_J $. A practical way to do so is based on the Euler-Maruyama (EM) scheme
\begin{align}
\label{eq: Euler scheme}
     \widehat{X}_{j+1} = \widehat{X}_j &+ b(\widehat{X}_j, t_j) \Delta t + \sigma(\widehat{X}, t_j) \sqrt{\Delta t}\, \zeta_{j+1},
\end{align}
where $\Delta t \coloneqq t_{j+1} - t_j$ is the step-size and 
\begin{equation}
    \zeta_{j+1} \coloneqq \frac{W_{t_{j+1}} - W_{t_j}}{\sqrt{\Delta t}} \sim \mathcal{N}(0,\operatorname{Id})
\end{equation}
is a standard normally distributed random variable. It can be shown that $\widehat{X}_{j}$ convergences to $X_{j \Delta t}$ in an appropriate sense \cite{kloeden1992stochastic}.
This readily leads to discrete versions of the quantities $\Delta_u$ and $S_u$ as defined in~\eqref{eq: estimator discr}.
Note that the discrete process $\widehat{X}$ is initialized at the random value $\widehat{X}_1 = \xi$ and $J$ is chosen according to the randomly drawn initial time $\tau$. More precisely, we set $J = \lceil   (T-\tau) / \Delta t  \rceil$ and use a smaller final step-size in order to arrive at the terminal time $T$. An alternative strategy would be to fix $J$ for all realizations and change the step-size $\Delta t = \frac{T - \tau}{J}$ depending on the value of $\tau$. We display the memory requirements of our methods as a function of the step-size $\Delta t$ in \Cref{fig: gpu mem}. 

\begin{figure}[t!]
  \centering
  \includegraphics[width=\linewidth]{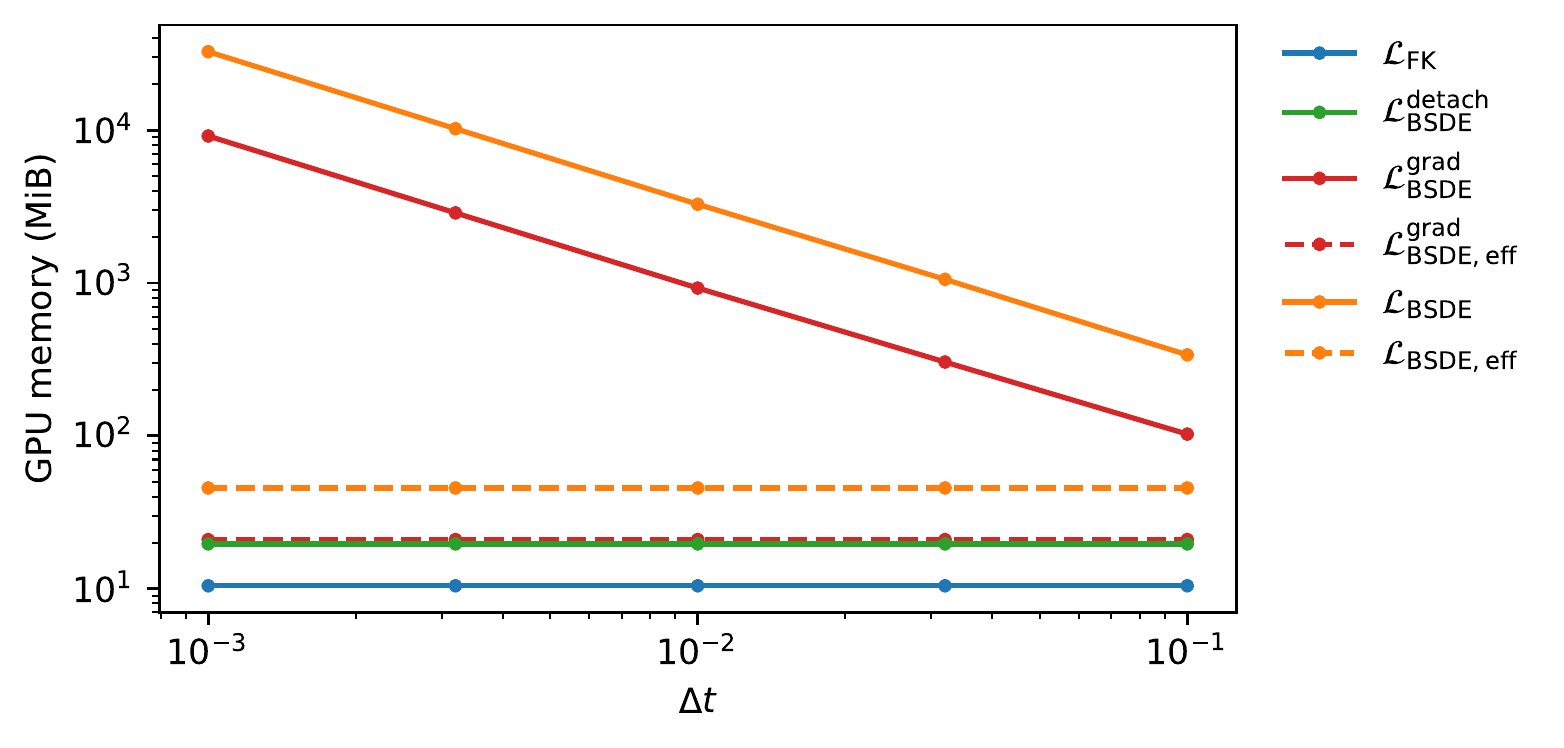}
  \vspace{-0.7cm}
  \caption{GPU memory requirements for a gradient step with batch size $K = 1024$ and different step-sizes $\Delta t$ for the SDE discretization when solving the heat equation. If we do not consider the efficient versions, the memory usage for the losses $\mathcal{L}^{\mathrm{grad}}_{\mathrm{BSDE}}$ and $\mathcal{L}_{\mathrm{BSDE}}$ depends approximately linearly on $\Delta t$.}
  \label{fig: gpu mem}
\end{figure}

In~\Cref{table:hp} we summarize the hyperparameters for our experiments. We specify the amount of steps (Schedule A) or time (Schedule B) that we train and the percentile, i.e., milestone, where we decrease the learning rate $\lambda$ and the step-size $\Delta t$. In order to have comparable results, we set a GPU memory limit of $8$ GiB during training and always use batch sizes $K \in \{2^4, 2^7, 2^{10}, 2^{13}, 2^{16}, 2^{19}\}$.

We employ a Multilevel neural network architecture, which has shown to be advantageous over standard feed-forward architectures. More precisely, we use an architecture with $L=3$ levels, amplifying factor $q=5$ for the HJB equation and $q=3$ for the other PDEs, intermediate residual connections ($\chi=1$) and without normalization layer, as defined in~\citet{berner2020numerically}. Note that in case of the losses $\mathcal{L}_{\mathrm{BSDE}}^{\mathrm{grad}}$ and $\mathcal{L}_{\mathrm{BSDE,\,eff}}^{\mathrm{grad}}$ we use the same architecture with output dimension $d$ for the neural network representing the function $r$ in~\eqref{eq: BSDE loss extra model}. As the activation function needs to be twice differentiable for certain losses, we replace the ReLU by the SiLU (also known as swish) activation function~\cite{hendrycks2016gaussian}.

Finally,~\Cref{table:hp} also specifies the number of samples used to approximate the MSE metrics and the number of batches used to estimate the loss and gradient variances. For the latter, we compute the value and derivative of $\widehat{\mathcal{L}}^{(K)}$ (w.r.t.\@ the parameters $\theta$ of $\Phi_\theta$) for $30$ batches consisting of $K$ independent samples of $(\xi,\tau,W)$ by performing forward and backward passes without updating the parameters $\theta$. To compute the MSE we evaluate $\widehat{\mathcal{L}}^{(N)}_{\mathrm{Eval}}$
with $N=10\cdot 2^{17}$ i.i.d.\@ samples drawn from the distribution of $(\xi,\tau)$ -- independently of the training data and independently for each evaluation. In this sense, we always evaluate our model w.r.t.\@ to the real solution $V$ on unseen data. For comparing the gradients we replace $(u(\xi,\tau)-V(\xi,\tau))^2$ by $\|\nabla_x u(\xi,\tau)-\nabla_x V(\xi,\tau)\|^2$ or $\|r(\xi,\tau)-\nabla_x V(\xi,\tau)\|^2$ in case of the methods $\mathcal{L}_{\mathrm{BSDE}}^{\mathrm{grad}}$ and $\mathcal{L}_{\mathrm{BSDE,\,eff}}^{\mathrm{grad}}$. When no closed-form solution for $V$ is available, as for the Black--Scholes model in \Cref{sec: black-scholes}, we use the version in~\eqref{eq:mse_mc} with another $2^{16}$ independent samples to estimate the inner expectation.

\begin{table}[t!]
    \centering
    \caption{Training and evaluation setup}
    \begin{tabular}{ll}
\toprule
\textbf{Schedule A: Step limit} & \\
steps $M$ & $3 \cdot 10^{4}$ \\
milestone & 0.9 \\
\midrule
\textbf{Schedule B: Time limit} & \\
time & $24h$ \\
milestone & 0.5 \\
\midrule
\textbf{Training} & \\
learning rate $\lambda$ & $[5\cdot 10^{-4}, 5\cdot 10^{-6}]$ \\
step-size $\Delta t$ &  $[10^{-2}, 10^{-3}]$ \\
optimizer & Adam \\
batch size $K$ & $\{2^4, 2^7, 2^{10}, 2^{13}, 2^{16}, 2^{19}\}$ \\
GPU memory constraint & $8$ GiB \\
\midrule
\textbf{Network} & \\
architecture & Multilevel \\
$(L,q,\chi)$ &         $\{(3,3,1),(3,5,1)\}$\\
activation function & SiLU \\
\midrule 
\textbf{Evaluation} & \\
samples $N$  &        $10 \cdot 2^{17}$     \\
batches (for variances) & 30 \\
\bottomrule
\end{tabular}
    \label{table:hp}
\end{table}

\section{Further Numerical Experiments}

Figures \ref{fig: heat plot} to \ref{fig: hjb scaling} show the results of additional numerical experiments. The corresponding captions describe the respective settings.

\begin{figure}[t!]
  \centering
  \includegraphics[width=1.0\linewidth]{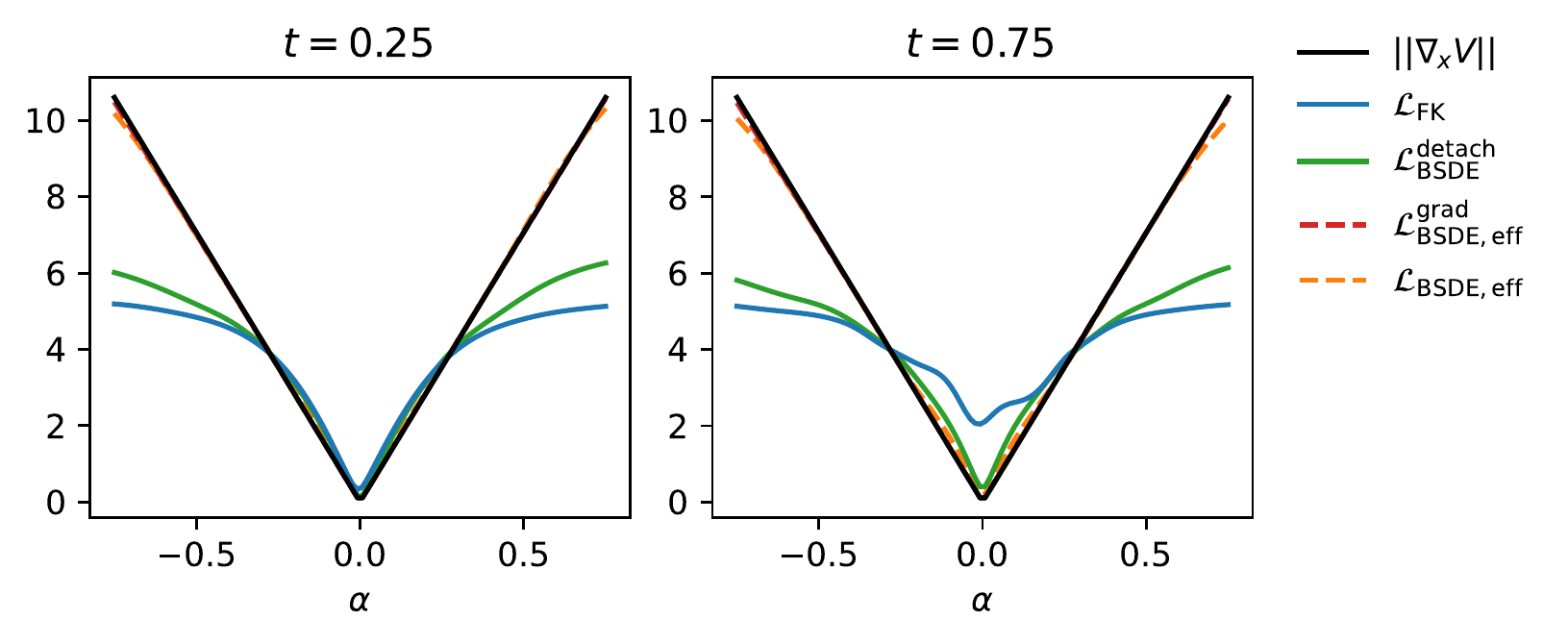}
  \vspace{-0.7cm}
  \caption{Approximation of the norm of the gradient $\|\nabla_x V\|$ of the solution $V$ to the heat equation in Section~\ref{sec: heat equation} for $t \in \{0.25,0.75\}$ and $x=(\alpha,\dots,\alpha)^\top$ after training for $30k$ steps with batch size $K = 1024$. As expected, the stochastic integral, which is present in all losses except $\mathcal{L}_{\mathrm{FK}}$, induces better approximation of the gradient. Furthermore, explicitly modelling the gradient as in
  $\mathcal{L}^{\mathrm{grad}}_{\mathrm{BSDE,\,eff}}$ or back-propagating the derivative of the neural network as in $\mathcal{L}_{\mathrm{BSDE,\,eff}}$ also improves the approximation outside of the sampling interval for $x$, i.e. outside of $[-0.5,0.5]^{50}$.}
  \label{fig: heat plot}
\end{figure}

\begin{figure}[t!]
  \centering
  \includegraphics[width=1.0\linewidth]{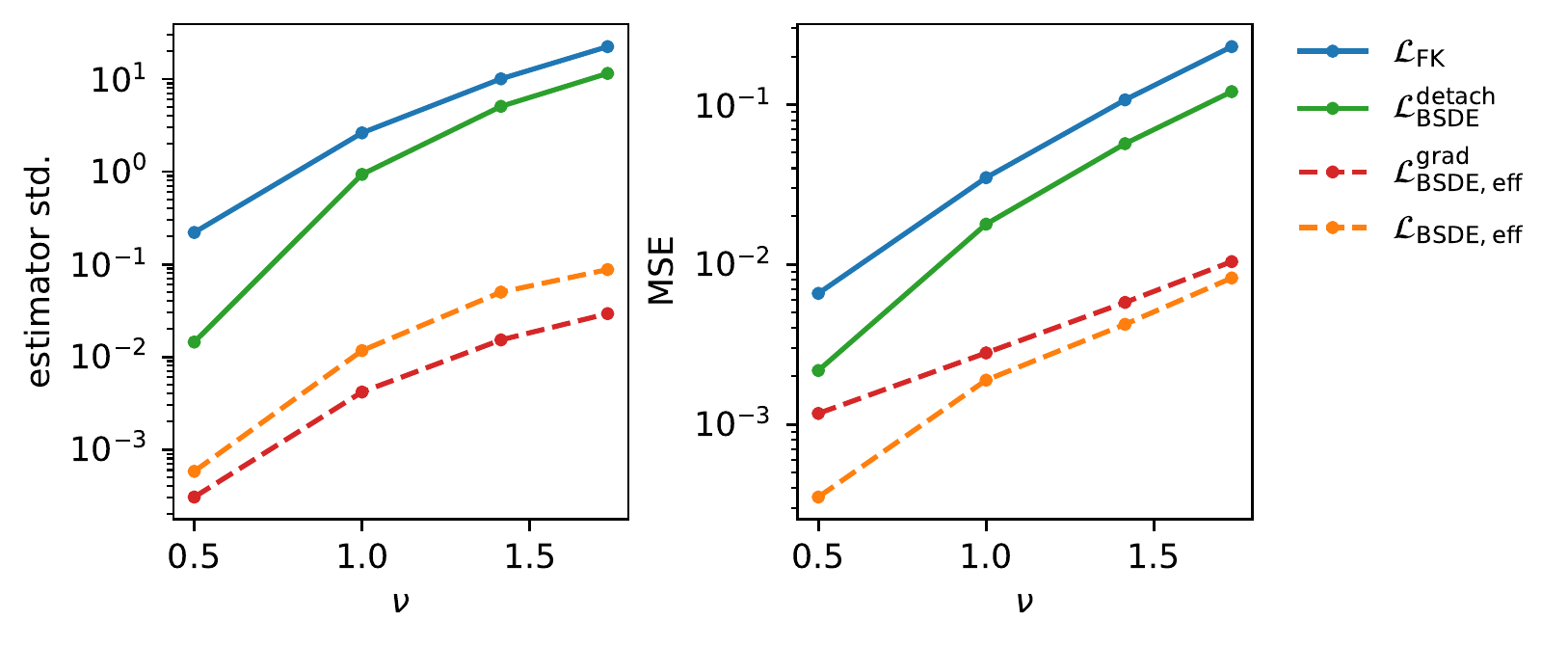}
  \vspace{-0.7cm}
  \caption{Estimator standard deviation and MSE after $30k$ gradient steps for the heat equation in Section~\ref{sec: heat equation} with batch size $1024$ and varying diffusivities $\nu$. While solving the PDE becomes more challenging for higher values of $\nu$, our proposed methods consistently outperform the baseline.}
  \label{fig: diff scaling}
\end{figure}

\begin{figure}[t!]
  \centering
  \includegraphics[width=1.0\linewidth]{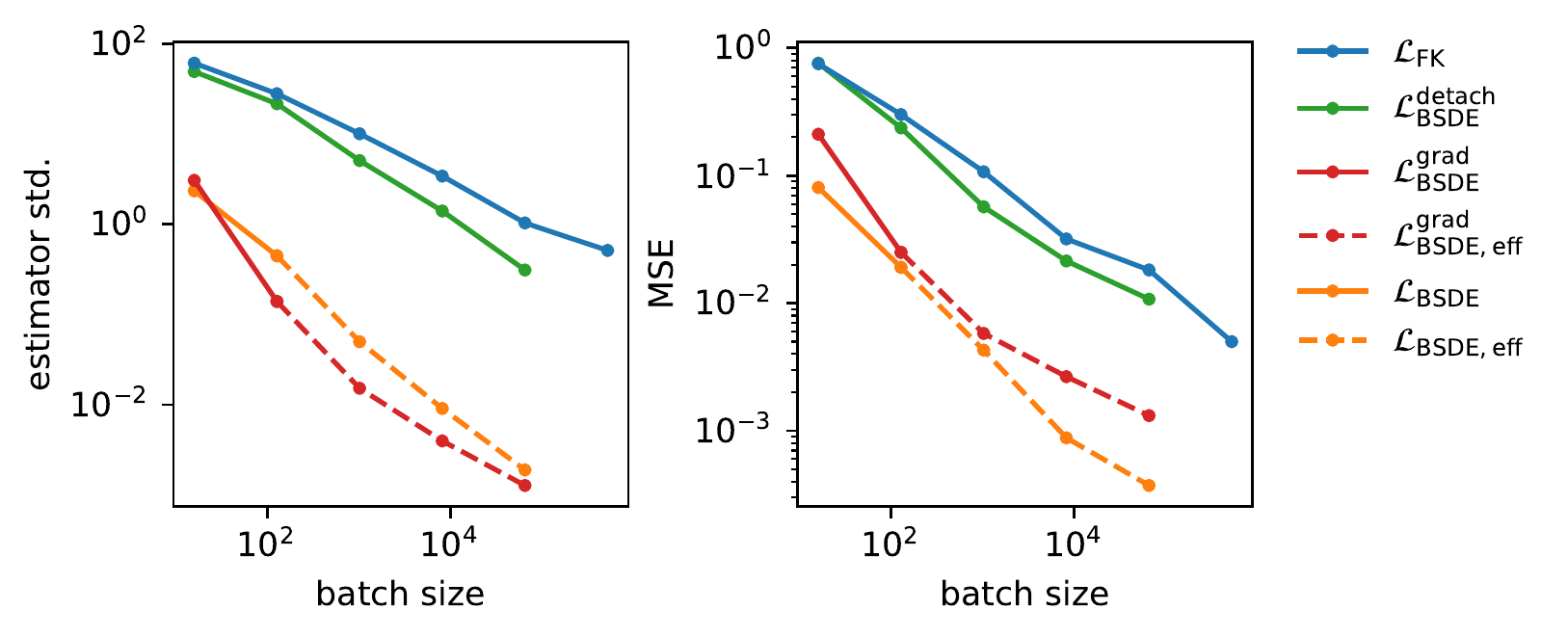}
  \vspace{-0.7cm}
  \caption{We observe performance differences similar to \Cref{fig: loss std first page} when solving the heat equation from~\Cref{sec: heat equation} with a higher diffusivity of $\nu=\sqrt{3}$.}
  \label{fig: heat diff scaling}
\end{figure}

\begin{figure}[t!]
  \centering
  \includegraphics[width=1.0\linewidth]{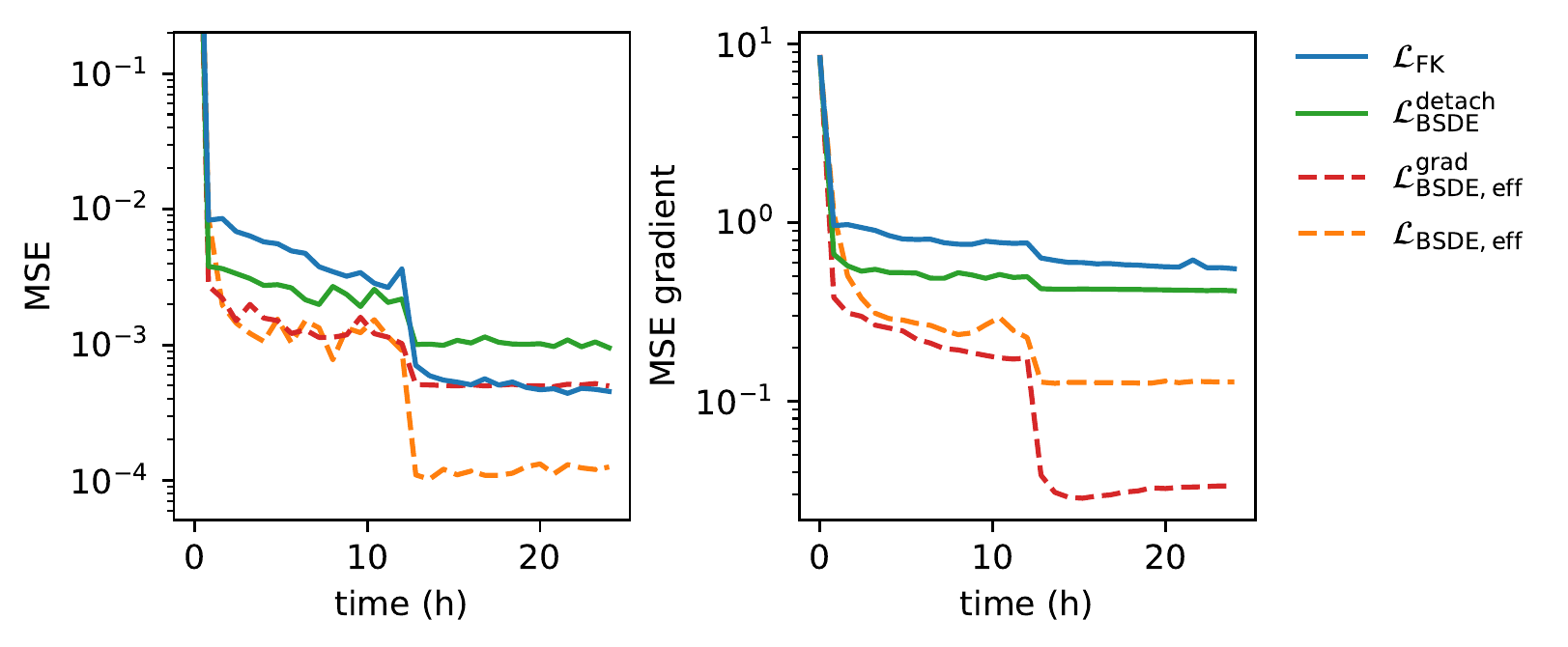}
  \vspace{-0.7cm}
  \caption{MSE of different losses as a function of the training time when solving the heat equation from \Cref{sec: heat equation} with batch size $K = 8192$. Although the time needed for computing and back-propagating the stochastic integral leads to significantly less gradient steps and thus samples of $(\xi,\tau,W)$ (see \Cref{fig: heat time scaling}), the loss $\mathcal{L}_{\mathrm{BSDE,\,eff}}$ still outperforms the loss $\mathcal{L}_{\mathrm{FK}}$.}
  \label{fig: heat time}
\end{figure}

\begin{figure}[t!]
  \centering
  \includegraphics[width=1.0\linewidth]{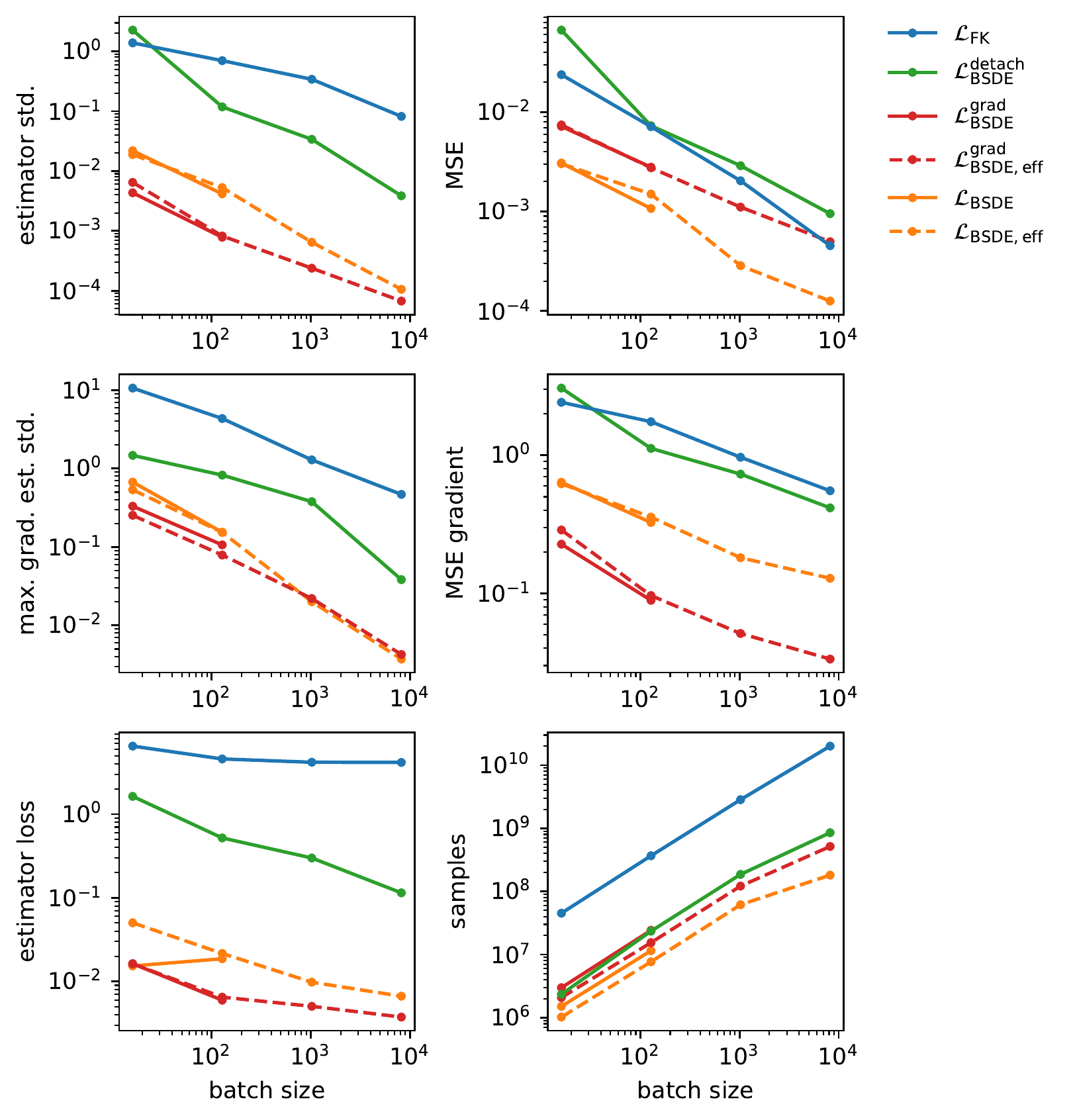}
  \vspace{-0.7cm}
  \caption{Scaling of estimator loss, standard deviation of the (gradient) estimator, and MSE when training for $24h$. We observe similar effects as in Figures \ref{fig: loss std first page} and \ref{fig: grad std heat}, even though the number of samples (or, proportionally, the number of steps) is significantly higher for the loss $\mathcal{L}_{\mathrm{FK}}$. For batch sizes higher than $K = 8192$ a larger time budget is necessary.}
  \label{fig: heat time scaling}
\end{figure}

\begin{figure}[t!]
\centering
  \centering
  \includegraphics[width=1.0\linewidth]{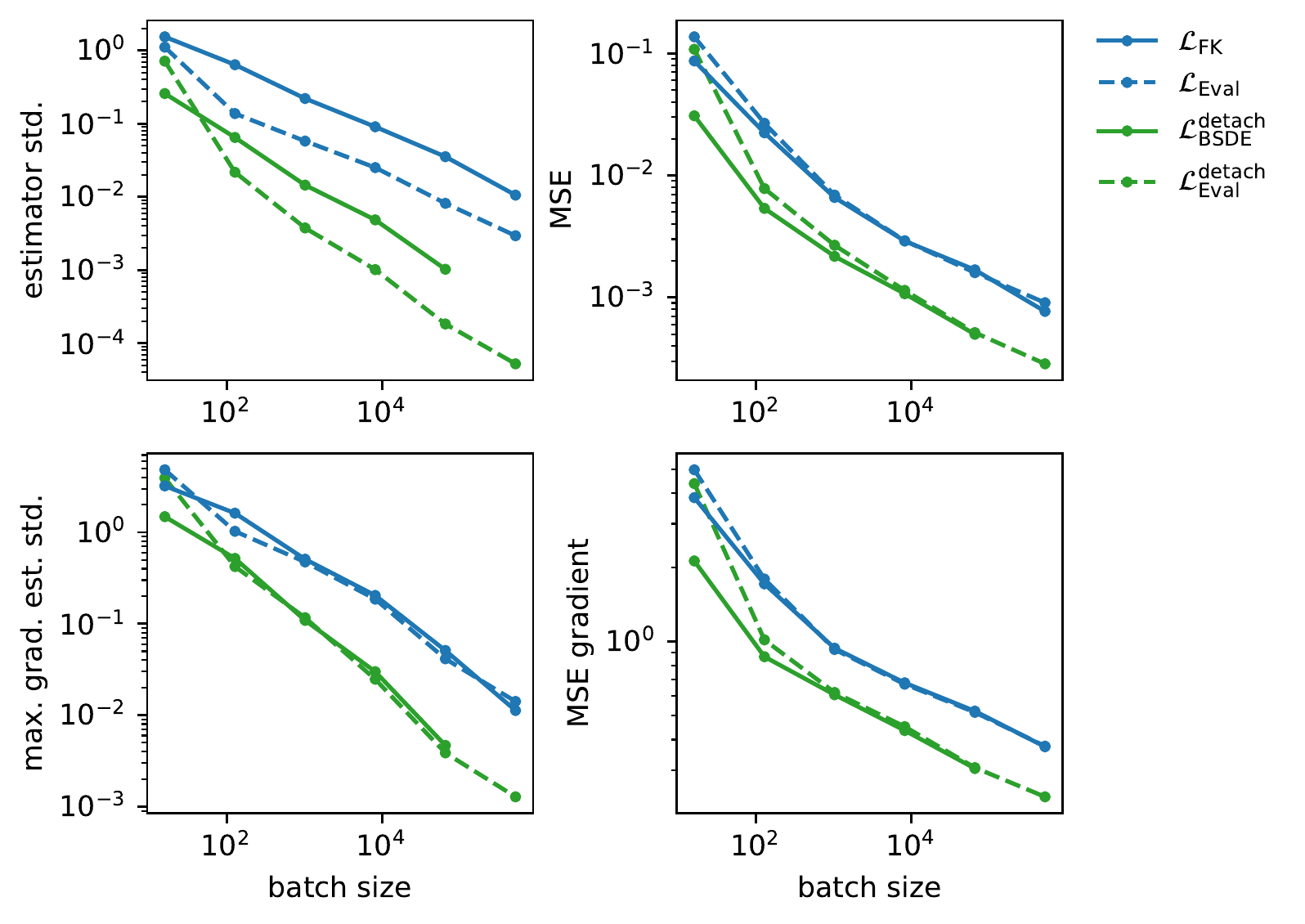}
  \vspace{-0.7cm}
  \caption{Performance and standard deviation of the losses $\mathcal{L}_{\mathrm{Eval}}$ and $\mathcal{L}_{\mathrm{Eval}}^\mathrm{detach}$ (detaching $S_u$ in \eqref{eq:eval_loss FK control variate} from the computational graph as in~\eqref{eq:detach}) for the heat equation from \Cref{sec: heat equation} compared to their natural counterparts. Note that they perform significantly worse for small batch sizes, but similarly for large batch sizes -- for the detached version one can allow for larger batch sizes, given a fixed memory budget. Note via a comparison to \Cref{fig: loss std first page,fig: grad std heat} that the efficient versions of the considered losses are still much better.}
  \label{fig: heat repeat all}
\end{figure}

\begin{figure}[t!]
  \centering
  \includegraphics[width=1.0\linewidth]{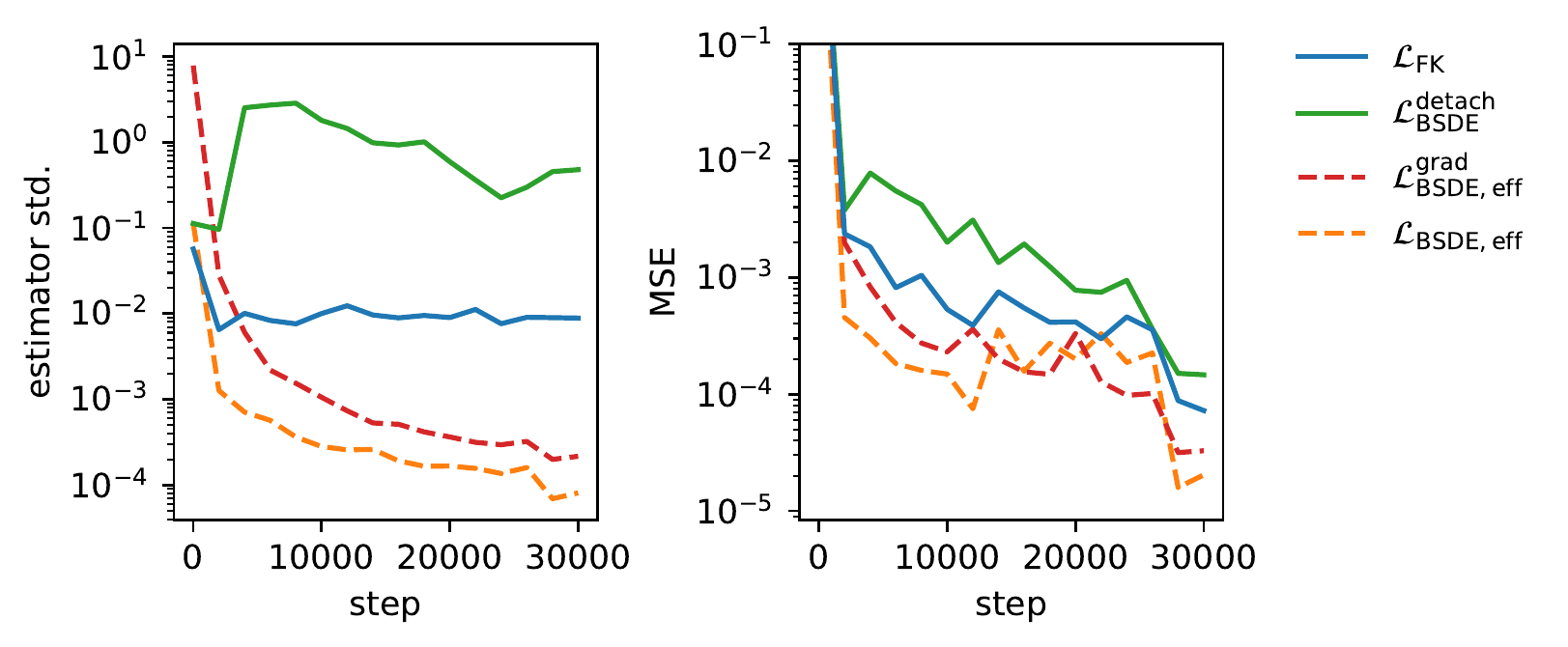}
  \vspace{-0.7cm}
  \caption{Estimator standard deviation
  and MSE as a function of the gradient steps when solving the Black--Scholes equation from \Cref{sec: black-scholes} with batch size $K = 8192$. In case of the loss $\mathcal{L}^{\mathrm{detach}}_{\mathrm{BSDE}}$, a bad initial approximation actually leads to a variance-increasing effect of the stochastic integral. As the error in the gradient is not back-propagated to the network parameters, this deteriorates the convergence.
  }
  \label{fig: bs detach}
\end{figure}

\begin{figure}[t!]
\centering
  \centering
  \includegraphics[width=1.0\linewidth]{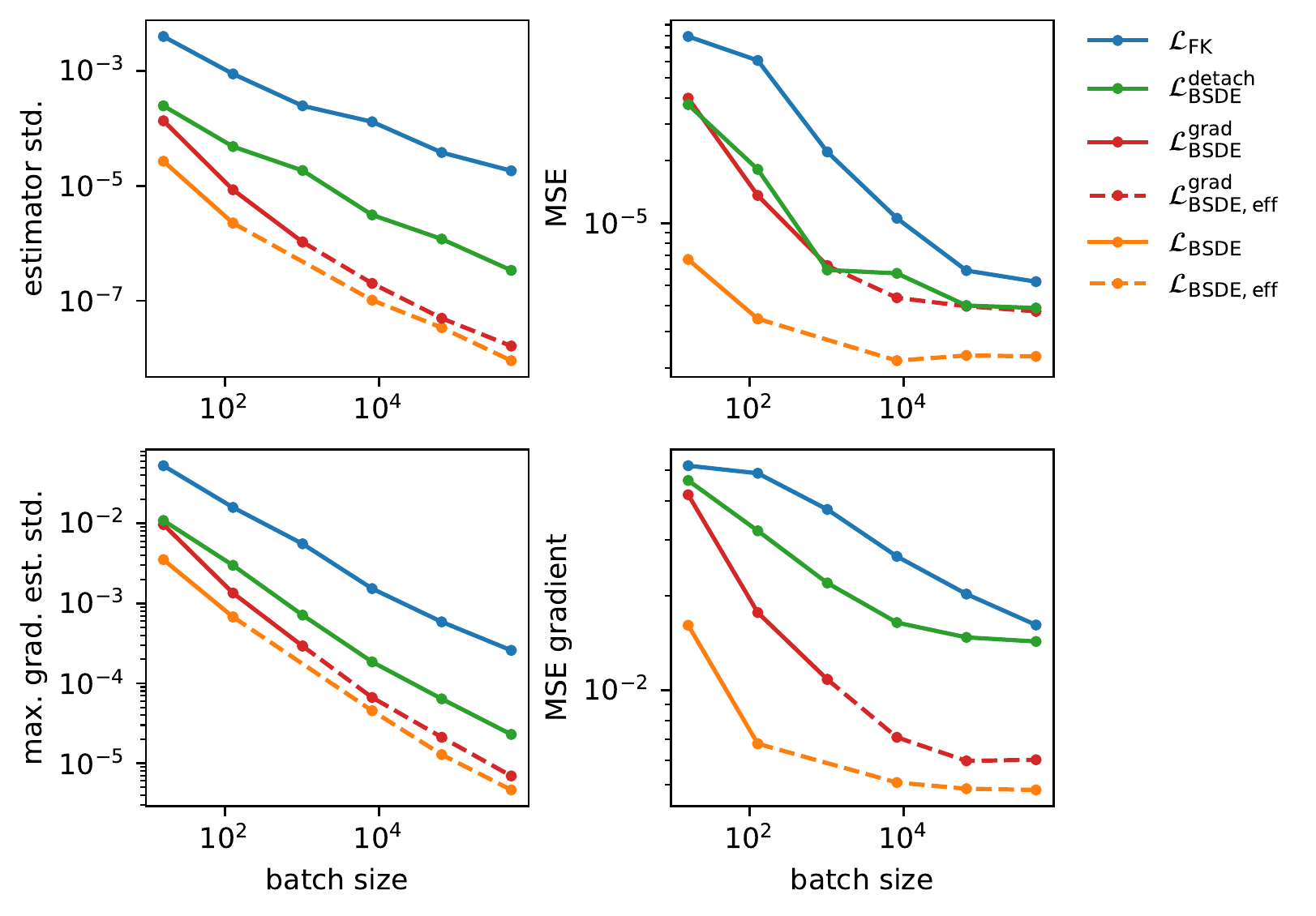}
  \vspace{-0.7cm}
  \caption{Similar to Figures \ref{fig: loss std first page},~\ref{fig: grad std heat}, and~\ref{fig: scaling bs}, our theoretical and empirical findings also hold in case of the HJB equation from~\Cref{sec: hjb}.}
  \label{fig: hjb scaling}
\end{figure}

\clearpage
\newpage 

\section{Additional Material}

The following Lemma details the relation of the HJB equation from \Cref{sec: hjb} and a linear PDE of Kolmogorov type.

\begin{lemma}[Linearization of HJB equation]
\label{lem: Linearization of HJB equation}
Let $\widetilde{V} \in C^{2, 1}(\R^d \times [0, T], \R)$ solve the Hamilton--Jacobi--Bellman equation in \eqref{eq: HJB PDE}, then $V = \exp(-\widetilde{V})$ fulfills the linear PDE as stated in \Cref{thm: general Feynman-Kac} with $k=0$ and $g = \exp(-\widetilde{g})$, i.e.
\begin{subequations}
\begin{align*}
    (\partial_t + L - f(x, t)) V(x, t) &= 0, & (x, t) \in \R^d \times [0, T), \\
    V(x, T) &= g(x), & x \in \R^d.
\end{align*}
\end{subequations}
Subsequently, we get the representation
\begin{equation}
    \widetilde{V}(x, t) = -\log \E\left[e^{-\mathcal{W}(X)}\Big| X_t = x\right]
\end{equation}
with $\mathcal{W}(X) = \int_t^T f(X_s, s)\mathrm ds + \widetilde{g}(X_T)$.
\end{lemma}

\begin{proof}
We consider the transformation $V = \exp(-\widetilde{V})$ and for notational convenience omit the space and time dependencies of $\widetilde{V}, f, b$, and $\sigma$. We can compute
\begin{align*}
    L e^{-\widetilde{V}} &= -b \cdot  e^{-\widetilde{V}}\nabla_x \widetilde{V} - \frac{1}{2} \sum_{i, j=1}^d (\sigma \sigma^\top)_{ij}\partial_{x_i} \big(\partial_{x_j} e^{-\widetilde{V}} \widetilde{V}  \big) \\ 
    \begin{split}
    &= - e^{-\widetilde{V}} \Bigg(b \cdot \nabla_x \widetilde{V}  + \frac{1}{2} \sum_{i, j=1}^d (\sigma \sigma^\top)_{ij}\partial_{x_i}  \partial_{x_j} \widetilde{V} \\
    &\qquad\qquad\qquad\qquad \   - \frac{1}{2} \sum_{i, j=1}^d (\sigma \sigma^\top)_{ij}\partial_{x_i} \widetilde{V} \partial_{x_j} \widetilde{V}\Bigg)\end{split} \\ 
    \begin{split}
    &= - e^{-\widetilde{V}} \Bigg(b \cdot \nabla_x \widetilde{V}  + \frac{1}{2} \sum_{i, j=1}^d (\sigma \sigma^\top)_{ij}\partial_{x_i}  \partial_{x_j} \widetilde{V} \\
    &\qquad\qquad\qquad\qquad \  -\frac{1}{2}  \sum_{i, j, k=1}^d \sigma_{ik}\sigma_{jk}\partial_{x_i} \widetilde{V} \partial_{x_j} \widetilde{V} \Bigg) \end{split} \\ 
    \begin{split}
    &= - e^{-\widetilde{V}} \Bigg(b \cdot \nabla_x \widetilde{V}  + \frac{1}{2} \sum_{i, j=1}^d (\sigma \sigma^\top)_{ij}\partial_{x_i}  \partial_{x_j} \widetilde{V} \\
    &\qquad\qquad\qquad\qquad \  - \frac{1}{2} \sum_{k=1}^d \left( \sum_{i=1}^d \sigma_{ik}\partial_{x_i} \widetilde{V} \right)^2 \Bigg) \end{split}\\ 
    \begin{split}
    &= -e^{-\widetilde{V}}\Big(L\widetilde{V} - \frac{1}{2}\|\sigma^\top \nabla_x \widetilde{V}\|^2\Big).\end{split}
\end{align*}
The PDE in \eqref{eq: general Feynman-Kac BVP} therefore becomes
\begin{subequations}
\begin{align}
    0 &= (\partial_t + L - f) e^{-\widetilde{V}} \\
    \begin{split}
    &= -e^{-\widetilde{V}}\Big((\partial_t + L) \widetilde{V} + f -\frac{1}{2}\|\sigma^\top \nabla_x \widetilde{V}\|^2 \Big),\end{split}
\end{align}
\end{subequations}
which is equivalent to the HJB equation in \eqref{eq: HJB PDE}.
\end{proof}

For a discussion on the application of \Cref{lem: Linearization of HJB equation} in the context of importance sampling of stochastic processes that are relevant in computational statistical physics and molecular dynamics we refer to \citet{hartmann2017variational}.

\end{document}